\documentclass[sn-mathphys-num]{sn-jnl}

\usepackage{graphicx}%
\usepackage{multirow}%
\usepackage{amsmath,amssymb,amsfonts}%
\usepackage{amsthm}%
\usepackage{mathrsfs}%
\usepackage[title]{appendix}%
\usepackage{xcolor}%
\usepackage{textcomp}%
\usepackage{manyfoot}%
\usepackage{booktabs}%
\usepackage{algorithm}%
\usepackage{algorithmicx}%
\usepackage{algpseudocode}%
\usepackage{listings}%
\usepackage{todonotes}
\usepackage{verbatim}
\usepackage{latexsym, epsfig}

\usepackage[utf8]{inputenc}
\usepackage[T1]{fontenc}
\usepackage{float} 
\usepackage{hhline}

\usepackage{url}
\usepackage{color}
\usepackage{array}
\usepackage{verbatim}
\usepackage{setspace}

\theoremstyle{thmstyleone}%
\newtheorem{theorem}{Theorem}
\theoremstyle{thmstyletwo}%
\newtheorem{remark}{Remark}

\theoremstyle{thmstylethree}%
\newcommand{\beq}{\begin{equation}}
\newcommand{\eeq}{\end{equation}}
\newcommand{\beqq}{\begin{equation*}}
\newcommand{\eeqq}{\end{equation*}}
\newcommand{\beqas}{\begin{eqnarray*}}
\newcommand{\eeqas}{\end{eqnarray*}}
\newcommand{\bsp}{\begin{split}}
\newcommand{\eesp}{\end{split}}

\newcommand{\bfa}[1]{\boldsymbol{#1}}

\usepackage{amscd}
\usepackage{colordvi}
\usepackage{pifont}
\usepackage{epstopdf}
\usepackage{hyperref}
\usepackage{subcaption}
\usepackage{cleveref}

\usepackage{comment}
\usepackage{upgreek}
\usepackage[mathscr]{eucal}

\usepackage{relsize}
\numberwithin{equation}{section}
\usepackage{enumitem}

\usepackage{epsfig}
\usepackage{booktabs}
\usepackage{siunitx}
\usepackage{xr}

\newlist{Assumption}{enumerate}{1}
\setlist[Assumption]{label=A\arabic*}

\definecolor{Blue}{rgb}{0,0,1}
\definecolor{Red}{rgb}{1,0,0}
\definecolor{Green}{rgb}{0,1,0}
\definecolor{Cyan}{rgb}{0,0.72,0.92}
\definecolor{Amethyst}{rgb}{0.6,0.4,0.8}
\definecolor{Bronze}{rgb}{0.8,0.5,0.2}
\definecolor{Violet}{rgb}{0.54,0.17,0.89}

\hypersetup{
	colorlinks=true,   
	citecolor=blue,     
	filecolor=blue,     
	linkcolor=red,    
	urlcolor=blue}      

\newlist{steps}{enumerate}{1}
\setlist[steps, 1]{label = Step \arabic*:}

\newcommand{\fullstate}{\bfa{x}}

\newcommand{\enc}{\phi_{\textnormal{e}}}
\newcommand{\dec}{\phi_{\textnormal{d}}}
\newcommand{\fullstateApprox}{\widetilde{\fullstate}}

\newcommand{\fullstateReconstruct}{\widehat{\fullstate}}
\newcommand{\err}{e}

\newcommand{\recon}{\textnormal{rec}}

\newcommand{\Jac}{\textnormal{Jac}}
\newcommand{\inte}{\textnormal{int}}

\def\norm#1{\|#1\|}

\begin{document}

\title[tLaSDI: Thermodynamics-informed latent space dynamics identification]{tLaSDI: Thermodynamics-informed latent space dynamics identification}


\author[1]{\fnm{Jun Sur R.} \sur{Park}}

\author[2]{\fnm{Siu Wun} \sur{Cheung}}

\author*[2]{\fnm{Youngsoo} \sur{Choi}}\email{choi15@llnl.gov}

\author*[3,4]{\fnm{Yeonjong} \sur{Shin}}\email{yeonjong$\_$shin@ncsu.edu}

\affil[1]{\orgdiv{Center for Artificial Intelligence and Natural Sciences}, \orgname{Korea Institute for Advanced Study (KIAS)}, \orgaddress{\city{Seoul}, \postcode{02455}, \country{Republic of Korea}}}

\affil[2]{\orgdiv{Center for Applied Scientific Computing}, \orgname{Lawrence Livermore National Laboratory}, \orgaddress{ \city{Livermore}, \postcode{94550}, \state{CA}, \country{USA}}}

\affil[3]{\orgdiv{Department of Mathematics}, \orgname{North Carolina State University}, \orgaddress{ \city{Raleigh}, \postcode{27695}, \state{NC}, \country{USA}}}

\affil[4]{\orgdiv{Mathematical Institute for Data Science}, \orgname{Pohang University of Science and Technology (POSTECH)}, \orgaddress{ \city{Pohang}, \postcode{37673}, \country{Republic of Korea}}}


\abstract{We propose a latent space dynamics identification method, namely tLaSDI, that embeds the first and second principles of thermodynamics. 
The latent variables are learned through an autoencoder as a nonlinear dimension reduction model.
The latent dynamics are constructed by a neural network-based model that precisely preserves certain structures for the thermodynamic laws through the GENERIC formalism.
An abstract error estimate is established, which provides a new loss formulation involving the Jacobian computation of autoencoder.
The autoencoder and the latent dynamics are simultaneously trained to minimize the new loss.
Computational examples demonstrate the effectiveness of tLaSDI, which exhibits robust generalization ability, even in extrapolation.
In addition, an intriguing correlation is empirically observed between a quantity from tLaSDI in the latent space and the behaviors of the full-state solution.
}




\maketitle

\section{Introduction}
\label{sec:intro}

For centuries, scientific research relied on 
established experiments, theories, and computational methods. 
First principles in physics, such as Newton's laws, 
Maxwell's equations, and the laws of thermodynamics, 
guided scientific exploration and 
fueled technological advancements in various disciplines. 
These principles were often derived or conceptualized by experimental or observational data. 
Leveraging advanced computing along with the first principles has played a pivotal role 
in understanding complex physical phenomena 
and solving challenging scientific and engineering problems.
With the advent of machine learning, a data-driven scientific computing approach has drawn huge attention.
The idea is to harness available data for scientific computing to complement or advance data-free approaches.
Many novel methods have been proposed in this regard, e.g. 
\cite{bongard2007automated,schmidt2009distilling,
brunton2016discovering,raissi2018hidden,wu2020data}. 


The reduced order models (ROMs) identify intrinsic low-dimensional structures for 
economically representing the solution in a high-dimensional space 
and formulate an underlying physical process for the simplified representation. 
ROMs have achieved significant success in many challenging physical models 
such as Navier-Stokes equations \cite{xiao2014non,burkardt2006pod}, 
Burgers' equation \cite{choi2019space,choi2020sns,carlberg2018conservative}, 
the Euler equations \cite{copeland2022reduced,cheung2022local}, 
shallow water equations \cite{zhao2014pod,cstefuanescu2013pod}, and 
Boltzmann transport problems \cite{choi2021space}. 
There are two major types of ROMs in building the latent dynamics. One is intrusive and the other is non-intrusive. 
The projection-based ROM is a typical intrusive approach that requires invasive changes to the source code of the high-fidelity physics solver, which is not always feasible. 
Yet, since the projection-based approach exploits the underlying physics, it yields robust extrapolation ability.
On the contrary, the non-intrusive ROM (e.g. dynamic mode decomposition \cite{schmid2010dynamic})
does not require access to the source code, which at the same time, makes it challenging to encode underlying physical laws into the model.
Thus, non-intrusive ROMs are typically regarded as black-box as it is not interpretable with existing scientific knowledge \cite{fries2022lasdi,he2023glasdi,bonneville2024gplasdi}. 
In building the latent space, there are two major constructions. One is the linear subspace (LS) and 
the other is the nonlinear manifold (NM).
The LS can be efficiently constructed 
by the singular value decomposition on the solution snapshot matrices.
However, it is limited by the assumption that 
the intrinsic solution space falls into a subspace with a small dimension, 
i.e., the solution space has a small Kolmogorov n-width. 
To address this limitation, the NM is proposed 
\cite{lee2020model,lee2021deep,kim2020efficient,kim2022fast}, 
which uses autoencoder for better representation ability and reconstruction accuracy.

The present work proposes a non-intrusive NM ROM method that possesses thermodynamic structures, namely, thermodynamics-informed latent space dynamics identification (tLaSDI).
The general equation for the 
non-equilibrium reversible-irreversible coupling (GENERIC) formalism
\cite{grmela1997dynamics,ottinger1997dynamics,ottinger2005beyond}
describes a general dynamical system beyond the equilibrium, which comprises of the reversible and irreversible components. 
The generators for the reversible and irreversible dynamics satisfy certain conditions, which ensure the first and second principles of thermodynamics.
tLaSDI uses the GENERIC formalism to design the latent dynamics 
and identifies the latent variables through an autoencoder \cite{demers1992non,hinton2006reducing}.
A distinct feature is that the latent dynamics obey the first and second laws of thermodynamics through the GENERIC formalism and thus yield energy and entropy functions in the latent space.
In particular, GENERIC formalism informed NNs (GFINNs) \cite{zhang2022gfinns} are employed as our design choice for the latent dynamics. GFINNs are NN-based models that exactly preserve the required structures with universal approximation properties.
Consequently, tLaSDI preserves the thermodynamic laws in the latent dynamics without invasive changes in the source codes.
In addition, we establish an abstract error estimate of the ROM approximation (Theorem~\ref{thm:total_err}), which characterizes all the error components in terms of the encoder, decoder, and latent dynamics. 
Based on the estimate, we derive a new loss formulation on which the encoder, decoder, and thermodynamic-informed latent dynamics are simultaneously trained.

The idea of imposing thermodynamic structures in latent dynamics is not new \cite{ottinger2015preservation,yu2021onsagernet,chen2023constructing,hernandez2021deep,moya2022physics,masi2022multiscale}.
Perhaps, the most relevant work to tLaSDI is \cite{hernandez2021deep}, which proposed an NN approach to design the latent dynamics from the GENERIC formalism. 
However, there are several major differences. 
One is the choice of the models for the latent dynamics.
\cite{hernandez2021deep} employs an NN-based model 
that does not satisfy the structural conditions exactly by construction, rather relies on an additional penalty term to enforce them.
Another is the loss formulation.
The standard loss consists of two terms - one matches the forward step prediction of the latent dynamics and the other deviates from the autoencoder. 
The proposed new loss function introduces additional two loss terms that involve Jacobian, which turns out to bring significant improvements in generalization ability.
The last one is the way the latent variables and dynamics are trained. \cite{hernandez2021deep} learns the latent variables separately, independent of latent dynamics. 
The latent dynamics are then trained later to fit the fixed latent variables. This separate training approach has been adopted in the literature, e.g., see \cite{moya2022physics,fries2022lasdi}. 
In contrast, tLaSDI uses the simultaneous training approach that trains the latent variables and their dynamics at the same time. This strategy has been employed in the literature as well, e.g., see \cite{champion2019data,vijayarangan2024data,he2023glasdi,bonneville2024gplasdi,chen2023constructing}.

Several numerical experiments are reported to demonstrate the performance of tLaSDI and verify the effectiveness of the new loss function.
Results show that tLaSDI exhibits robust extrapolation ability, which pure data-driven approaches typically lack.
Also, it is empirically observed that the latent entropy function manifests a correlated behavior aligned with the full-state solution.
For the parametric Burgers' equation, the rate of entropy production has the largest value at the final time when the solution exhibits the stiffest behavior.
In contrast, the entropy production rate for the heat equation behaves the opposite as it saturates as time goes on, which aligns with the solution behavior of the heat equation. See Figures~\ref{fig:BG_mean_std_S_dSdt_sol}
and ~\ref{fig:HT_mean_std_S_dSdt_sol}.


The rest of the paper is organized as follows. 
Section~\ref{sec:prob} describes the problem setup and some preliminaries.
The proposed method, tLaSDI, is presented in Section~\ref{sec:rom} along with the new loss formulation. 
Section~\ref{sec:error_estimates} is devoted to the abstract error estimate.
Numerical examples are presented in Section~\ref{sec:num} before the conclusions in Section~\ref{sec:conclusion}.

\section{Preliminaries and problem formulation}
\label{sec:prob}

Let us consider a dynamical system (the full-order model)
\beq
\label{eq:full}
\bsp
\dot{\bfa{x}}_\mu(t) &= F_\mu(\bfa{x}_\mu(t),t) \quad \forall t >0 \quad \text{with }
\bfa{x}_\mu(0) = \text{x}^0(\mu) \in \mathbb{R}^{N},
\end{split}
\eeq
defined in a high dimensional space $\mathbb{R}^{N}$, 
where $\mu$ represents some parameters of interest in $\Omega_\mu$
and $F_{\mu}$ is an appropriate field function that guarantees the existence of the solution to \eqref{eq:full} for all $\mu \in \Omega_\mu$.

Since the state dimension is very large, i.e., $N \gg 1$, a direct numerical simulation is computationally costly, prohibiting multiple simulations at varying $\mu$.
Even worse, numerical simulations are not possible if $F_{\mu}$ is unknown. 
The non-intrusive ROM can be used in this regard to achieve the following two goals.
\begin{itemize}
\item {\bf Computational savings: } When $F_{\mu}$ is fully known, the goal of ROM is to reduce the computational time in simulating the full-order dynamics at multiple $\mu$.
It requires the full-state solution data at selected parameters from direct numerical simulations of \eqref{eq:full}. 
The qualify of the ROM approximation is measured by the prediction accuracy on unseen parameters.
In particular, it is popularly used for parametric 
partial differential equations (PDE) problems.
\item {\bf Data-driven discovery of dynamical systems: } When $F_{\mu}$ is not available, 
the goal of ROM is to learn a dynamical system from data that produces trajectories that are similar to data (interpolation) or predict states in future times (extrapolation).
Here the trajectory data is assumed to be given. 
The pure data-driven approach, however, is known to lack robust extrapolation ability.
This is the task where the intrusive ROM is not applicable as $F_\mu$ is unknown.
\end{itemize}

In the ROM framework, the full-state variables $\bfa{x} \in \mathbb{R}^{N}$ are transformed to the so-called latent variables $\bfa{z}\in \mathbb{R}^{n}$ that lie in a much lower dimensional space, i.e., $n \ll N$. Such a mapping is called an encoder $\phi_\text{e}(\cdot): \mathbb{R}^{N} \to \mathbb{R}^{n}$.
Given an encoder, a latent dynamical system is formed
\beq
\label{eq:reduced}
\dot{\bfa{z}}_\mu(t) = F^r_\mu(\bfa{z}_\mu(t),t)
\quad \forall t >0 \quad \text{with }
\bfa{z}_\mu(0) = \phi_{\text{e}}({\bfa x}_\mu(0)) \in \mathbb{R}^n.
\eeq
Lastly, a decoder function $\phi_\text{d}(\cdot): \mathbb{R}^{n} \to \mathbb{R}^{N}$ transforms the latent variables back to the original full-order space, providing the ROM approximation to ${\bfa x}_\mu(t)$:
\begin{equation*}
\label{eq:recon_phid_X}
\widetilde{\bfa{x}}_\mu(t) = \phi_{\text{d}}(\bfa{z}_\mu(t)),
\end{equation*}
where $\bfa{z}_\mu$ is the solution to \eqref{eq:reduced}.

The constructions of the encoder, decoder, and latent dynamics determine a specific ROM method. 
In general, such constructions are obtained from data.
Here the data refers to a set of either direct numerical simulations of \eqref{eq:full} or a collection of spatiotemporal trajectories.
The present work focuses on developing a new NN-based model for latent dynamics to be learned from data while leveraging nonlinear dimension reduction via autoencoder.

\subsection{Neural networks}
\label{sec:fnn}
A $L$-layer feed-forward NN is a function  
$f_{\text{NN}}: \mathbb{R}^{d_{\text{in}}} \ni x \mapsto f^L(x) \in \mathbb{R}^{d_{\text{out}}} $ 
defined recursively according to 
\begin{equation*}
f^l(x) = W^l \sigma(f^{l-1}(x)) + b^l, \ \ 2\leq l \leq L, 
\end{equation*}
with $f^1(x) = W^1x+b^1$. The activation function $\sigma$ is a scalar function defined on $\mathbb{R}$ applied in element-wise. 
It is typically chosen to satisfy the conditions of the universal approximation theorem \cite{cybenko1989approximation,mhaskar1996neural,siegel2020approximation}. 
$W^l\in \mathbb{R}^{n_l\times n_{l-1}}$ 
and $b^l\in \mathbb{R}^{n_l}$ are the weight matrix and the bias vector of the $\ell$-th layer, respectively, where $n_0 = d_{\text{in}}$, $n_{L} = d_{\text{out}}$. 
The collection of all the weight matrices and bias vectors is denoted by $\bfa{\theta}$.
To explicitly acknowledge the dependency of $\bfa{\theta}$, the NN is often denoted as $f_{\text{NN}}(\cdot;\bfa{\theta})$.

In the parametric dynamical systems, 
it is often needed to let the NN parameters depend on $\mu \in \Omega_p$ to enhance the flexibility and prediction accuracy.
The hypernetwork \cite{ha2017hypernetworks} is a NN model designed for this purpose, which introduces another NN that takes $\mu$ as input and outputs the network parameter $\bfa{\theta}(\mu)$ of $f_{\text{NN}}$.
Accordingly, the hypernetwork is denoted by $f_{\text{NN}}(\cdot;\bfa{\theta}(\mu))$.


The autoencoder (AE) 
is an NN-based model designed for dimension reduction, which comprises of 
an encoder $\phi_{\text{e}}(\cdot;\bfa{\theta}_\text{e}): \mathbb{R}^{N}\to\mathbb{R}^{n}$ and a decoder $\phi_{\text{d}}(\cdot;\bfa{\theta}_\text{d}):\mathbb{R}^{n}\to\mathbb{R}^{N}$
where $\bfa{\theta}_\text{e}$ and 
$\bfa{\theta}_\text{d}$ are the NN parameters for the encoder and decoder respectively. 
The autoencoder takes high-dimensional inputs 
$\bfa{x}$ and transforms it to the so-called latent variables $\bfa{z}$ through the encoder.
The decoder then takes $\bfa{z}$ and returns high-dimensional variables $\widehat{\bfa{x}}$
having the same size as the original inputs.
Ideally, the autoencoder is trained to reconstruct the original inputs, i.e., 
$\bfa{x} \approx \widehat{\bfa{x}} = \phi_{\text{d}}(\phi_{\text{e}}(\bfa{x};\bfa{\theta}_\text{e});\bfa{\theta}_\text{d})$.

\subsection{GENERIC formalism}
The GENERIC formalism \cite{grmela1997dynamics,ottinger2005beyond} is a mathematical framework that encompasses both conservative and dissipative systems, providing a comprehensive description of beyond-equilibrium thermodynamic systems. 
The formalism describes the dynamical systems involving four functions -- the scalar functions $E$ and $S$ that represent the total energy and entropy of the system respectively, and the matrix-valued functions $L$ and $M$ called the Poisson and the friction matrices, respectively; 
\begin{equation}
\label{eq:GENERIC}
\begin{split}
&\dot{\bfa{z}} = L(\bfa{z}) \nabla E(\bfa{z}) +M(\bfa{z}) \nabla S(\bfa{z}), \\
\textrm{subject to} \ \ &L(\bfa{z}) \nabla S(\bfa{z}) = M(\bfa{z}) \nabla E(\bfa{z}) = \bfa{0}, \\
&L(\bfa{z}) \ \textrm{is skew-symmetric, i.e.,} \ L(\bfa{z}) = -L(\bfa{z})^\top,\\
&M(\bfa{z}) \ \textrm{is symmetric} \ \& \ \textrm{positive semi-definite.} 
\end{split}
\end{equation}


The degeneracy conditions guarantee the first and second laws of thermodynamics, i.e., the following properties of energy conservation and non-decreasing entropy, as
$\frac{d}{dt}E(\bfa{z}) = 0$ and $\frac{d}{dt}S(\bfa{z}) \geq 0$.

{\bf GFINNs:}
Several NN-based models were proposed to embed the GENERIC formalism \cite{hernandez2021structure,lee2021machine,zhang2022gfinns}.
In particular, we briefly review the model proposed in \cite{zhang2022gfinns}, namely, GFINNs.
We direct readers to refer to \cite{zhang2022gfinns} for more details.
The GFINNs comprise of four NNs -- $E_{\text{NN}}$, $S_{\text{NN}}$, $L_{\text{NN}}$ and $M_{\text{NN}}$ -- that represent NN approximations to $E$, $S$, $L$ and $M$ in \eqref{eq:GENERIC}, respectively. 
These NNs are designed to exactly satisfy the degeneracy conditions in \eqref{eq:GENERIC} while also being adequately expressive to capture the underlying dynamics from data.
The GFINN construction is based on the following theorem \cite{zhang2022gfinns}. 
\begin{theorem}[Lemma 3.6 of \cite{zhang2022gfinns}] \label{lemma:skewQ}
Let $A_j$ be a skew-symmetric matrix of size $n \times n$, $j = 1,\dots, m$ and $g(\cdot): \mathbb{R}^{n} \to \mathbb{R}$ a differentiable scalar function.
For the matrix valued function $Q_g(\cdot): \mathbb{R}^{n} \to \mathbb{R}^{m \times n}$ whose $j$-th row is defined as $(A_j \nabla g(\cdot))^\top$, we have $Q_g(\bfa{z}) \nabla g(\bfa{z}) = \bfa{0}$ for all $\bfa{z} \in \mathbb{R}^{n}$. 
\end{theorem}


Let $S_{\textnormal{NN}}$ be a feed-forward NN
and consider the associated matrix-valued function $Q_{S_{\textnormal{NN}}}$ according to Theorem \ref{lemma:skewQ}.
The NN model $L_{\textnormal{NN}}$ for the Poisson matrix is then constructed by 
\begin{equation*}
\label{eq:ann_structure}
L_{\textnormal{NN}}(\bfa{z}) := (Q_{S_{\textnormal{NN}}}(\bfa{z}))^\top \, U_{\textnormal{NN}}(\bfa{z}) \, Q_{S_{\textnormal{NN}}}(\bfa{z}),
\end{equation*}
where $U_{\textnormal{NN}}(\bfa{z}): \mathbb{R}^{n} \to  \mathbb{R}^{m \times n}$ is a skew symmetric matrix-valued NN.
It can be readily checked that the degeneracy condition and the symmetry of the matrix functions are exactly satisfied by the above construction.
$M_{\text{NN}}$ and $E_{\text{NN}}$ are similarly built. 
Furthermore, such construction enjoys a universal approximation property as shown in \cite{zhang2022gfinns}.


\section{Thermodynamics-informed learning for latent space dynamics}
\label{sec:rom}
We propose the thermodynamics-informed latent dynamics identification (tLaSDI) method, which designs the latent space dynamics as an NN-based model that embeds the first and second laws of thermodynamics through the GENERIC formalism.

In particular, we propose to design the latent dynamics $F^r_{\mu}$ in the ROM \eqref{eq:reduced} 
from the GENERIC formalism. 
While some alternatives \cite{hernandez2021deep,lee2021machine} are equally applicable, we confine ourselves to GFINNs \cite{zhang2022gfinns} for the sake of simplicity of discussion. 
For latent manifold learning, 
we propose the hyper-autoencoder that combines the hypernetworks \cite{ha2017hypernetworks,lee2022hyperdeeponet} and the autoencoder to further improve the performance in the parametric cases.
That is, the hyper-encoder $\phi_{\text{e}}(\cdot;\bfa{\theta}_{\text{e}}(\mu))$ and the hyper-decoder $\phi_{\text{d}}(\cdot;\bfa{\theta}_{\text{d}}(\mu))$ are employed. 
The schematic description of the model for tLaSDI is illustrated in Figure~\ref{fig:flowchart_prediction}.

 \begin{figure}[ht]
 \centering
\includegraphics[scale=0.11]{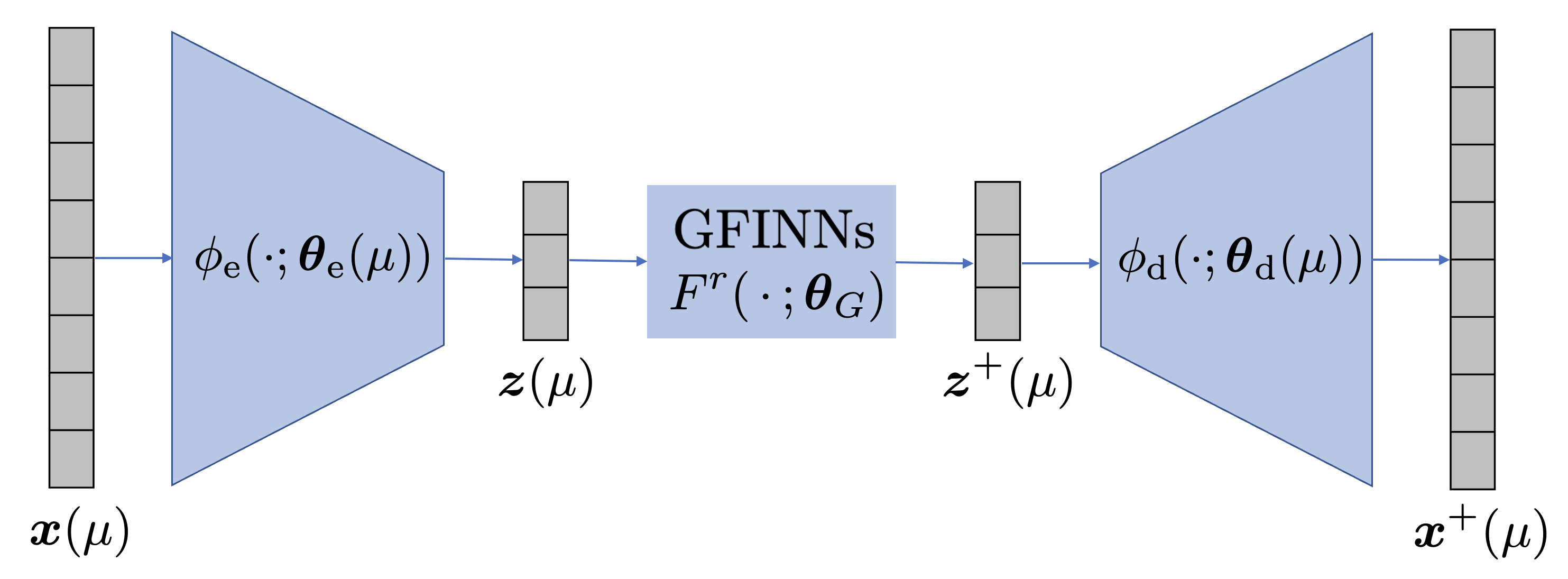}
\caption{The schematic NN model of tLaSDI. The hyper-autoencoder is used for the parametric dynamical systems. The latent dynamics are modelled by GFINNs.}
\label{fig:flowchart_prediction}
\end{figure}  

\begin{remark}
    \cite{hernandez2021structure} is the first work that proposed an NN-based model to encode the GENERIC formalism, namely, Structure-Preserving NN (SPNN). While SPNN aims to enforce the degeneracy conditions through an additional loss term that penalizes these conditions, in principle, it does not satisfy them exactly. 
    \cite{lee2021machine} proposed a model, namely, GNODE, which satisfies the degeneracy conditions exactly without any additional loss term, however, does not have universal approximation properties.
    The NN model of \cite{zhang2022gfinns}, namely, GFINNs is the one that comes with the exact degeneracy conditions and the universal approximation properties.
\end{remark}

\subsection{tLaSDI: Loss formulation}
\label{sec:training_objectives}

We present a novel loss formulation for tLaSDI, which is derived from an abstract error estimate presented in Section~\ref{sec:error_estimates}.

Let $\Gamma_{\text{train}}$ be the set of training parameters. 
For a given parameter $\mu \in \Gamma_{\text{train}}$,
let $\{\bfa{x}^k_\mu:=\bfa{x}_\mu(t_k)\}_{k\ge 0}$ be a collection of full-state trajectory data/snapshots.
Assuming the time step is sufficiently small, one may also obtain the (approximated) derivatives $\{\dot{\bfa{x}}^k_\mu\}_{k \ge0}$.
For example, the central difference formula gives 
$\dot{\bfa{x}}^k_\mu \approx (\bfa{x}^{k+1}_\mu -  \bfa{x}^{k-1}_\mu)/2\Delta t$.
While the approximation errors exist, for simplicity of notation, we denote the approximated derivatives by $\dot{\bfa{x}}^k_\mu$ with slight abuse of notation.

We then propose to train the tLaSDI to minimize the loss function defined by
\begin{equation}
\label{eq:loss_tLaSDI_GFINNs}
{\mathcal L}(\bfa{\theta}) = \lambda_{\text{int}}{\mathcal L}_{\text{int}}(\bfa{\theta}) + \lambda_{\text{rec}}{\mathcal L}_{\text{rec}}(\bfa{\theta}) +
\lambda_{\text{Jac}}{\mathcal L}_{\text{Jac}}(\bfa{\theta}) +\lambda_{\text{mod}}{\mathcal L}_{\text{mod}}(\bfa{\theta}),
\end{equation}
where $\bfa{\theta} = \{\bfa{\theta}_{\text{e}}, \bfa{\theta}_{\text{d}}, \bfa{\theta}_{\text{G}}\}$ is the collection of all the trainable parameters.
The loss function consists of four components
-- ${\mathcal L}_{\text{int}}, {\mathcal L}_{\text{rec}}, {\mathcal L}_{\text{Jac}}, {\mathcal L}_{\text{mod}}$ -- the subscripts stand for `\underline{int}egration', `\underline{rec}onstruction', `\underline{Jac}obian' and `\underline{mod}el', respectively.
The first two terms ${\mathcal L}_{\text{int}}, {\mathcal L}_{\text{rec}}$ are commonly employed in the context of the data-driven discovery of dynamical systems \cite{hernandez2021deep,vijayarangan2024data,moya2022physics} and the last term ${\mathcal L}_{\text{mod}}$ was introduced in \cite{champion2019data,he2023glasdi,conti2023reduced}.
The Jacobian loss term ${\mathcal L}_{\text{Jac}}$ is newly proposed in the present work.
The error estimate in Section~\ref{sec:error_estimates} provides a theoretical justification for the use of the loss function \eqref{eq:loss_tLaSDI_GFINNs}.

\begin{remark}
    The latter two loss terms utilize the derivative information. When the derivative of data is not available, while the fourth term shall be dropped, the third term can be revised to avoid using the derivatives. See the details below.
\end{remark}


{\bf The integration loss} is a commonly employed term for learning problems in dynamical systems. 
It measures the discrepancy between the latent state from the encoder at $\bfa{x}^{k+1}_\mu$ and the corresponding one-time step integration of the latent dynamics starting at the latent state from the encoder at the previous time $\bfa{x}^{k}_\mu$. That is, the term enforces 
\begin{align*}
    \phi_{\text{e}}(\bfa{x}^{k+1}_\mu)
    \approx  \phi_{\text{e}}(\bfa{x}^{k}_\mu) + \int_{t_k}^{t_{k+1}} F_{\mu}^r(\bfa{z}_\mu(t),t) dt.
\end{align*}
By summing it over all the data, the integral loss is defined by 
\begin{equation*}
{\mathcal L}_{\text{int}} = 
\sum\limits_{\mu \in \Gamma_{\text{train}}} \sum\limits_{k} \left\|
\phi_{\text{e}}(\bfa{x}^{k+1}_\mu) -\phi_{\text{e}}(\bfa{x}^{k}_\mu) - \int_{t_k}^{t_{k+1}} F_{\mu}^r(\bfa{z}_\mu(t),t) dt
\right\|_2^2,
\end{equation*}
where 
$t_k$ represents the time after $k$-th time steps
and the integral is approximated by employing a numerical integrator (e.g., Runge-Kutta methods).
It is worth noting that the integration loss is independent of the decoder.

{\bf The reconstruction loss} is another standard loss term when it comes to training autoencoder. 
It measures the discrepancy between the full-state data and the autoencoder reconstruction. 
The term particularly takes into account the autoencoder's ability to reconstruct the high-dimensional original data and is defined by
\begin{align*}
    {\mathcal L}_\recon =  
\sum\limits_{\mu \in \Gamma_{\text{train}}} \sum\limits_{k} 
\norm{
\bfa{x}^k_\mu -
\phi_{\text{d}}\circ \phi_{\text{e}}(\bfa{x}^{k}_\mu)}_2^2.
\end{align*}

{\bf The Jacobian loss} is the term derived from the abstract error estimates of the ROM approximation (Section~\ref{sec:error_estimates}).
The term requires the (approximated) derivative data of the full-state dynamics and measures the time derivative of the reconstruction error, i.e., 
\begin{align*}
    \frac{d}{dt}\left(\bfa{x}^k_\mu -
\phi_{\text{d}}\circ \phi_{\text{e}}(\bfa{x}^{k}_\mu)\right) = \left(I - J(\bfa{x}^k_\mu)\right)\dot{\bfa{x}}^k_\mu,
\end{align*}
where  $J(\bfa{x}^k_\mu) := J_\textnormal{d}(\phi_{\text{e}}(\bfa{x}^{k}_\mu)) J_\textnormal{e}(\bfa{x}^k_\mu)$ is the Jacobian of the autoencoder and $I$ is the identity matrix.
The Jacobian loss is then given by
\begin{equation} 
    {\mathcal L}_{\text{Jac}} = \sum\limits_{\mu \in \Gamma_{\text{train}}}\sum\limits_{k}
\norm{
\left(I - J(\bfa{x}^k_\mu)\right)\dot{\bfa{x}}^k_\mu}_2^2.
\end{equation}
While the loss term may be viewed as a heuristic Sobolev-type loss for the reconstruction loss (which introduces a high-order derivative in the loss), our motivation for introducing the Jacobian loss stems from Theorem \ref{thm:total_err} as an effort to minimize an upper bound of the error.

In the case where accurate derivative data are not available, one may revise the Jacobian loss by observing
\begin{equation*}
    \norm{
\left(I - J(\bfa{x}^k_\mu)\right)\dot{\bfa{x}}^k_\mu}_2
\le \norm{
I - J(\bfa{x}^k_\mu)}_2 \cdot \|\dot{\bfa{x}}^k_\mu\|_2
\le \norm{
I - J(\bfa{x}^k_\mu)}_F \cdot \|\dot{\bfa{x}}^k_\mu\|_2,
\end{equation*}
where $\|\cdot\|_F$ is the Frobenius norm.
The revised Jacobian loss may be given by 
\begin{align*}
    {\mathcal L}_{\text{Jac}} = \sum\limits_{\mu \in \Gamma_{\text{train}}}\sum\limits_{k}
\norm{
I - J(\bfa{x}^k_\mu)}_F^2,
\end{align*}
which does not require the derivative information.

{\bf The model loss} may be viewed as a complement term to the Jacobian loss. Since the dimension of the latent variables is smaller than the full-state dimension, the Jacobian of the autoencoder can never be a full rank, which implies $J(\cdot) \ne I$. 
Therefore, one may design another loss from the following observation.
\begin{align*}
    \|\left(I - J(\bfa{x}^k_\mu)\right)\dot{\bfa{x}}^k_\mu \|_2
    &\le \norm{
\dot{\bfa{x}}^k_\mu- 
J_\textnormal{d}(\phi_{\text{e}}(\bfa{x}^{k}_\mu))\, F_{\mu}^r \circ \phi_{\text{e}}(\bfa{x}^{k}_\mu)
}_2 \\
&\quad+ 
\|J_\textnormal{d}(\phi_{\text{e}}(\bfa{x}^{k}_\mu))\| \cdot 
\norm{J_\textnormal{e}(\bfa{x}^k_\mu) \dot{\bfa{x}}^k_\mu - 
F_{\mu}^r \circ \phi_{\text{e}}(\bfa{x}^{k}_\mu)
}_2.
\end{align*}

The first term on the right-hand side measures the modeling error of the full-order dynamics. Ideally, since $\dot{\bfa{x}}_\mu^k= F(\bfa{x}_\mu^k)$, the corresponding NN construction via the latent space dynamics is 
$J_\textnormal{d}(\phi_{\text{e}}(\bfa{x}^{k}_\mu))\, F_{\mu}^r \circ \phi_{\text{e}}(\bfa{x}^{k}_\mu)$.
This means that the first term measures how well the NN construction approximates the underlying full-order dynamics through the derivative data.

The second term on the right-hand side involves with $\norm{J_\textnormal{e}(\bfa{x}^k_\mu) \dot{\bfa{x}}^k_\mu - 
F_{\mu}^r \circ \phi_{\text{e}}(\bfa{x}^{k}_\mu)
}_2$, which can be interpreted as the modeling error of the latent dynamics.
If the encoder were given, $J_\textnormal{e}(\bfa{x}^k_\mu) \dot{\bfa{x}}^k_\mu$ describes the underlying latent space dynamics derived from the full-order dynamics $\dot{\bfa{x}}_\mu^k= F(\bfa{x}_\mu^k)$.
Since $F_{\mu}^r \circ \phi_{\text{e}}(\bfa{x}^{k}_\mu)$ represents our model for the latent dynamics, the discrepancy between the two terms measures the approximation (or modeling) error of the latent space dynamics via $F_\mu^r$. 
Altogether leads to the model loss defined by
\begin{equation*}
\begin{split}
{\mathcal L}_{\text{mod}} = \sum\limits_{\mu \in \Gamma_{\text{train}}} \sum\limits_{k}
\norm{J_\textnormal{e}(\bfa{x}^k_\mu) \, \dot{\bfa{x}}^k_\mu - F_{\mu}^r\circ \phi_{\text{e}}(\bfa{x}^{k}_\mu)}_2^2
+  \norm{
\dot{\bfa{x}}^k_\mu- 
J_\textnormal{d}(\phi_{\text{e}}(\bfa{x}^{k}_\mu))\, F_{\mu}^r \circ \phi_{\text{e}}(\bfa{x}^{k}_\mu)
}_2^2.
\end{split}
\end{equation*}

Lastly, we note that, unlike the Jacobian loss, the model loss requires one to have the derivative information.

The Jacobian and model loss terms require the computation of Jacobian matrices, which can be efficiently evaluated by leveraging the Jacobian-vector product (JVP) feature provided by PyTorch \cite{paszke2019pytorch} or JAX \cite{bradbury2018jax}. A simple Pytorch snippet for JVP is presented to illustrate its simplicity:
\begin{verbatim}
        J_e = torch.autograd.functional.jvp(encoder,x,dx)[1]
        J_d = torch.autograd.functional.jvp(decoder,z,dz)[1]
\end{verbatim}

The presented framework trains the encoder, decoder, and latent dynamics simultaneously to minimize the loss function \eqref{eq:loss_tLaSDI_GFINNs}.
    This simultaneous training has been used in the literature \cite{champion2019data,vijayarangan2024data,he2023glasdi,bonneville2024gplasdi,chen2023constructing}.
    Since tLaSDI imposes a thermodynamic structure on the latent state dynamics, this comes as a natural choice because the latent variables found by the autoencoder shall be compatible with the imposed structure (thermodynamics).
    We, however, acknowledge other works (e.g. \cite{hernandez2021deep,moya2022physics,fries2022lasdi}) that train the autoencoder separately as a data preprocessing and then learn the latent dynamics from the fixed latent states.
    Empirically, we found that tLaSDI performs significantly better when it is trained simultaneously. 
    This may indicate that the imposed underlying structure on the latent dynamics may play a key role in this regard, which facilitates the interaction between the autoencoder and the NN-based latent dynamics.

\section{Abstract error estimate}
\label{sec:error_estimates}
We present an abstract error estimate, which reveals all the components attributing the total errors of the ROM approximation.
For ease of discussion, without loss of generality, we suppress all the parametric dependencies.

Let $\bfa{x}(t)$ be the solution to the full-order model of $\dot{\bfa{x}} = F(\bfa{x})$.
We are interested in estimating the ROM approximation error, i.e.,
\begin{align*}
    e(t;t_0) := \bfa{x}(t) - \phi_{\text{d}}
    \left( \phi_{\text{e}}(\bfa{x}(t_0)) + 
    \int_{t_0}^{t} F^r(\bfa{z}(s)) ds \right),
\end{align*}
where $\bfa{z}$ is the solution of the latent space dynamics \eqref{eq:reduced} satisfying 
$\bfa{z}(t_0) = \phi_{\text{e}}(\bfa{x}(t_0))$.

\begin{theorem} \label{thm:total_err}
Let $\bfa{x}$ be the solution to the full-order model $\dot{\bfa{x}} = F(\bfa{x})$,
and let $\bfa{z}$ be the solution to the latent dynamics of $\dot{\bfa{z}} = F^r(\bfa{z})$ with 
$\bfa{z}(t_0) = \phi_\text{e}(\bfa{x}(t_0))$.
Suppose the Jacobian of the decoder $\phi_\text{d}$ is Lipschitz continuous and bounded, and $F^r$ is bounded.
For any $t > t_0$, the ROM error is bounded by 
$$
\| \err(t;t_0) \| \lesssim 
\varepsilon_\inte(t;t_0) +
\varepsilon_\recon(t;t_0) +
\varepsilon_\Jac(t;t_0) +
\varepsilon_\textup{mod}(t;t_0),
$$
where each error term is defined by
\begin{equation*}
\begin{split}
\varepsilon_\inte(t;t_0) &= \int_{t_0}^t \|\phi_\textnormal{e}(\fullstate) - \bfa{z}\| ds,\\
\varepsilon_\recon(t;t_0) & = \| \fullstate(t_0) - (\phi_{\textnormal{d}} \circ \phi_{\textnormal{e}})(\fullstate(t_0)) \| +  \| \fullstate(t) - (\phi_{\textnormal{d}} \circ \phi_{\textnormal{e}})(\fullstate(t)) \|, \\
\varepsilon_\Jac(t;t_0) &= \int_{t_0}^t \|(I - J(\fullstate(s))) \dot{\fullstate}(s) \| ds, \\
\varepsilon_\textup{mod}(t;t_0) & = \int_{t_0}^t \bigg(\| J_\textnormal{e}(\fullstate)\dot{\fullstate}  - F^r(\bfa{z}) \| + \|\dot{\fullstate} - J_\textnormal{d}(\bfa{z})F^r(\bfa{z}) \| \bigg) ds,
\end{split}
\end{equation*}
and $\lesssim$ hides unimportant constants.
\end{theorem}
\begin{proof}
Let $\fullstateReconstruct(t) = (\dec \circ \enc)(\fullstate(t))$, which can be interpreted as the reconstructed full-state dynamics from the autoencoder. 
Since $\dot{\fullstate}(t) = F(\fullstate(t))$, it follows from the chain rule that the dynamics of $\dot{\fullstate}$ is described by
\begin{equation*}\label{eq:fom-approx}
\begin{split}
\dot{\fullstateReconstruct} = J_\textnormal{d}(\enc(\fullstate)) \ J_\text{e}(\fullstate)  \ F(\fullstate) \quad \forall t > 0, \quad 
\fullstateReconstruct(t_0) = (\dec \circ \enc)(\fullstate(t_0)). 
\end{split} 
\end{equation*}
Since the dynamics of $\fullstateReconstruct$ require the original full-state dynamics $F$, it is not available in practice. Yet, this serves as an idealized model to be learned for the ROM.
Let $e_\text{ideal}(t):=  \fullstate(t) - \fullstateReconstruct(t) $.
It then can be checked that 
the ideal error $e_\text{ideal}$ is governed by 
\begin{equation*}\label{eq:err-fom}
\begin{split}
\dot{e}_\text{ideal} = (I - J(\fullstate)) F(\fullstate) \quad \forall t, \quad 
e_\text{ideal}(t_0) = \fullstate(t_0) - (\dec \circ \enc)(\fullstate(t_0)). 
\end{split} 
\end{equation*}
By integrating over time, we obtain 
$e_\text{ideal}(t) = e_\text{ideal}(t_0) + \int_{t_0}^t (I - J(\fullstate(s)))F(\fullstate(s)) ds$,
which gives 
\begin{equation*}
 \|e_\text{ideal}(t)\| \le \alpha  \left( \|e_\text{ideal}(t_0)\| + \int_{t_0}^t \|(I - J(\fullstate(s))) \dot{\fullstate}(s) \| ds \right)+ (1-\alpha) \|e_\text{ideal}(t)\|,
 \end{equation*}
for any $\alpha \in (0,1)$.


Let $\fullstateApprox(t) = \phi_\text{d} \left(\phi_\text{e}(\fullstate(t_0)) + \int_{t_0}^t F^r(\bfa{z}(s)) ds\right)$
where $\bfa{z}$ is the solution to $\dot{\bfa{z}} = F^r(\bfa{z})$ with $\bfa{z}(t_0) = \phi_\text{e}(\fullstate(t_0))$.
We note that $\fullstateApprox$ is the one constructed from the ROM framework that uses both the autoencoder and the latent dynamics.
Let $e_{\text{ROM}}(t) = \fullstateApprox(t) - \fullstateReconstruct(t) $ be the error between the ROM reconstruction and the ideal reconstruction. 
It then can be checked that the error is governed by 
\begin{align*}
    \dot{e}_\text{ROM} = J(\fullstate) F(\fullstate) - J_\text{d}(\bfa{z})F^r(\bfa{z}) \quad \forall t, \quad e_\text{ROM}(t_0) = 0.
\end{align*}
For any $\alpha \in (0,1)$, since 
\begin{align*}
    \| J(\fullstate)F(\fullstate)  - J_\text{d}(\bfa{z})F^r(\bfa{z}) \| 
    &\le 
    \alpha \|J_\text{d}(\enc(\fullstate))\|\cdot \| J_\text{e}(\fullstate) F(\fullstate)  - F^r(\bfa{z}) \| \\
    &\quad+ 
    \alpha \| J_\textnormal{d}(\enc(\fullstate)) - J_\textnormal{d}(\bfa{z})\|\cdot \|F^r(\bfa{z})\| \\
    &\quad 
    +(1-\alpha)\|(I-J(\fullstate))F(\fullstate)\|
    \\
    &\quad
    +(1-\alpha)\|F(\fullstate) - J_\text{d}(\bfa{z})F^r(\bfa{z}) \|,
\end{align*}
it follows from the Lipschitz continuities of $\phi_\text{d}$, $J_\text{d}$ and the boundedness of $F^r$ that 
\begin{align*}
    \|e_\text{ROM}(t)\| &\lesssim 
    \int_{t_0}^t\| J_\text{e}(\fullstate) F(\fullstate)  - F^r(\bfa{z}) \| ds +
    \int_{t_0}^t \|\phi_\text{e}(\fullstate) - \bfa{z}\| ds
    \\
    &\quad 
    + \int_{t_0}^t \|(I-J(\fullstate))F(\fullstate)\| +\|F(\fullstate) - J_\text{d}(\bfa{z})F^r(\bfa{z}) \| ds.
\end{align*}
Therefore, by combining the above estimates, we obtain
\begin{align*}
    \|\bfa{x}(t) - \fullstateApprox(t)\|
    &\lesssim 
    \int_{t_0}^t \|\phi_\text{e}(\fullstate) - \bfa{z}\| ds + \|e_\text{ideal}(0)\| + \|e_\text{ideal}(t)\|\\
    &\quad+ \int_{t_0}^t \|(I - J(\fullstate(s))) \dot{\fullstate}(s) \| ds
    \\
    &\quad 
    + \int_{t_0}^t \bigg(\| J_\text{e}(\fullstate) F(\fullstate)  - F^r(\bfa{z}) \| + \|F(\fullstate) - J_\text{d}(\bfa{z})F^r(\bfa{z}) \| \bigg) ds,
\end{align*}
which completes the proof.
\end{proof}

For a small time interval $[t_k,t_{k+1}]$, 
if $\bfa{z}$ satisfies $\bfa{z}(t_k) = \phi_\text{e}(\bfa{x}(t_k))$, we have
\begin{align*}
    \varepsilon_\inte(t_{k+1};t_k) &= \int_{t_k}^{t_{k+1}} \|\phi_\text{e}(\fullstate) - \bfa{z}\| ds
    \approx \frac{\Delta t}{2} \cdot \|\phi_\text{e}(\fullstate(t_{k+1})) - \bfa{z}(t_{k+1})\|,
\end{align*}
where the integral is approximated by the trapezoidal rule. 
Since $\dot{\bfa{z}} = F^r(\bfa{z})$, 
it can be checked that 
\begin{align*}
    \varepsilon_\inte(t_{k+1};t_k) \approx 
    \frac{\Delta t}{2} \cdot
    \left\|\phi_\text{e}(\fullstate(t_{k+1})) - \phi_\text{e}(\fullstate(t_{k})) - \int_{t_k}^{t_{k+1}} F^r(\bfa{z}(s)) ds \right\|,
\end{align*}
and from which, one can discover the integration loss $\mathcal{L}_\text{int}$.
By following a similar argument, all the loss terms of \eqref{eq:loss_tLaSDI_GFINNs} can be discovered from the corresponding error components of Theorem~\ref{thm:total_err}.

Since the result of Theorem~\ref{thm:total_err} is general and abstract, it could be used for estimating the error of any ROM approximation.
For example, suppose $F(\bfa{x}) = M\bfa{x}$ is linear where $M = [Q,\tilde{Q}]\begin{bmatrix}
    \Sigma & \\ & \tilde{\Sigma}
\end{bmatrix}\begin{bmatrix}
    Q^\top \\ \tilde{Q}^\top
\end{bmatrix}$ is a spectral decomposition, and the encoder and decoder are also linear, say, $\enc(\bfa{x}) = Q^\top \bfa{x}$
and $\dec(\bfa{z}) = Q\bfa{z}$ with $Q^\top Q = I$.
Since the projection-based LS-ROM uses $F^r(\bfa{z}) = Q^\top M Q\bfa{z} = \Sigma \bfa{z}$ as the latent dynamics,  
it can be checked that 
\begin{align*}
    \phi_\text{e}(\bfa{x}) - \bfa{z} &= Q^\top \bfa{x} - \bfa{z}, \\
    (I - J(\bfa{x}))\dot{\bfa{x}} &= (I - QQ^\top)M\bfa{x}= \tilde{Q}\tilde{\Sigma}\tilde{Q}^\top \bfa{x}, \\
    F(\bfa{x}) - J_\text{d}(\bfa{z})F^r(\bfa{z})
    &= Q\Sigma (Q^\top \bfa{x} - \bfa{z}) + \tilde{Q}\tilde{\Sigma}\tilde{Q}^\top \bfa{x},
    \\
    J_\text{e}(\bfa{x})F(\bfa{x}) - F^r(\bfa{z}) &= \Sigma (Q^\top \bfa{x} - \bfa{z}).
\end{align*}
Since $Q^\top \bfa{x} = \bfa{z}$ as it matches the initial condition, the error estimate of Theorem~\ref{thm:total_err} yields 
\begin{align*}
    \|e(t;t_0)\| \lesssim \|\tilde{Q}\tilde{Q}^\top\bfa{x}(t_0)\| + \int_{t_0}^t \|\tilde{Q}\tilde{\Sigma}\tilde{Q}^\top \bfa{x}\| ds.
\end{align*}

The error estimate provides a theoretical foundation for the proposed loss function \eqref{eq:loss_tLaSDI_GFINNs} and explains the role of each loss component.
Some loss terms were introduced in \cite{champion2019data,he2023glasdi} in a heuristic way, while 
Theorem~\ref{thm:total_err} provides a solid theoretical justification for them.
The advantage of the new loss will be demonstrated through numerical experiments in Section~\ref{sec:num}.

\section{Numerical examples}
\label{sec:num}
This section presents numerical examples to demonstrate the effectiveness of tLaSDI. 
Two learning tasks are considered. One is the data-driven discovery of dynamical systems to predict the trajectories in future times (extrapolation). The other task focuses on the parametric PDE problems.

All the implementations were done on a Livermore Computing Lassen system's NVIDIA V100 (Volta) GPU, located at the Lawrence Livermore National Laboratory. This GPU is equipped with 3,168 NVIDIA CUDA Cores and possesses 64 GB of GDDR5 GPU Memory. 
The source codes are written in the open-source PyTorch \cite{paszke2019pytorch}. The codes will be published in a GitHub page. 

\subsection{Data-driven discovery of dynamical systems: Extrapolation}
For the data-driven discovery of dynamics, we consider the non-parametric full-order dynamics, which corresponds to $|\Omega_\mu|=1$.
The extrapolation predicts the full-state solution at times outside of the training range. 
For numerical experiments, we divide the given trajectory data of the time range $[0,T+\delta]$ into two.
One contains the trajectory of $[0,T]$, which is used for training.
The other contains the remaining one of $(T,T+\delta]$, which is used for testing.
The extrapolation accuracy is measured by the averaged relative ${\ell}_2$ error over time: 
\beq
\label{eq:rel_l2_nonpara}
e^{{\ell}_2}_{\bfa{x}} = \frac{1}{\vert \kappa_{\text{test}} \vert}\sum\limits_{k \in \kappa_{\text{test}}}\left(\frac{\norm{\bfa{x}^k-\widetilde{\bfa{x}}^k}_{2}}{\norm{\bfa{x}^k}_{2}}\right),
\eeq
where $\kappa_{\text{test}} = \{k\ |\ t_k \in (T, T+\delta] \}$
and $\widetilde{\bfa{x}}^k$ is the ROM prediction at time $t_k$ computed with the initial condition of $x_\mu(T)$.
\eqref{eq:rel_l2_nonpara} is referred to as the extrapolation error.



tLaSDI is compared with two NM ROM methods. 
One is the Vanilla-FNN, which uses plain feed-forward NNs for the encoder, decoder, and latent dynamics and trains them on the standard loss function which contains the first two loss terms of \eqref{eq:loss_tLaSDI_GFINNs}.
TA-ROM \cite{hernandez2021deep} uses the sparse autoencoder and then models the latent dynamics using SPNN \cite{hernandez2021deep,hernandez2021structure}. This method follows the separate training schemes and uses the standard loss function.
Other details of Vanilla-FNN and TA-ROM can be found in Appendix~\ref{app:overview-other-method}.

NN architectures are chosen to be comparable with each other in terms of the number of NN parameters. 
\texttt{Adam} \cite{kingma2014adam} optimizer is employed throughout for the training.

\subsubsection{Couette flow of an Oldroyd-B fluid}
\label{sec:oldroyd}
The Couette flow of an Oldroyd-B fluid model describes viscoelastic fluids composed of linear elastic dumbbells representing polymer chains in a solvent.
The model involves four state variables for all $100$ nodes.
The state variables for the $i$-th node are the fluid's position on each mesh node ($q_i$), its velocity in the $x$-direction ($v_i$), internal energy ($e_i$), and the shear component ($\tau_i$) of the conformation tensor.

Following \cite{hernandez2021structure,hernandez2021deep}, we construct the full-order state by concatenating all the variables, i.e.,
$\bfa{x}(t) =  [\bfa{q}(t) \ \bfa{v}(t) \ \bfa{e}(t) \ \bfa{\tau}(t) ]^\top \in \mathbb{R}^{400}$,
where $\bfa{q} = [q_1 \dots q_{100}]$,
$\bfa{v} = [v_1  \dots v_{100}]$,
$\bfa{e} = [e_{1} \dots e_{100}]$,
and $\bfa{\tau} = [ \tau_{1} \dots \tau_{100}]$.
Consequently, the dimension of the full-order model is $400$. 
We set $T = 0.9$, $\delta = 0.1$ with the fixed time step of $\Delta t = \frac{1}{150}$.

In all the methods, the autoencoders contain $183K$ trainable parameters with the latent space dimension of $n=8$. 
For the latent space dynamics, Vanilla-FNN, TA-ROM, and tLaSDI consist of $143K$, $135K$, and $139K$ parameters, respectively.
All methods were trained for approximately $5000$ seconds of wall-clock time resulting in around $40K$ and $280K$ iterations for tLaSDI and Vanilla-FNN respectively. 
For TA-ROM, SAE and SPNN undergo training for $53K$ and $360K$ iterations respectively, leading to a total of $413K$ iterations. 
The number of iterations for SAE and SPNN of TA-ROM is adjusted to enhance the performance.
Other implementation details can be found in Appendix \ref{sec:app_VC}.


Figure~\ref{fig:VC_loss_figures1} shows the training loss and extrapolation error trajectories by the three different methods with respect to the wall-clock time in seconds.
We run 10 independent simulations and report all the training trajectories on the left.
It can be seen that the training losses are separated by each method. 
This is expected as each method has its own training loss and they scale differently.
Regardless, we see that all the losses saturate at the end of the training. 
Since TA-ROM sequentially trains SAE and SPNN in that order, the part before the black vertical dashed line corresponds to the SAE training and the rest falls into the SPNN training.
It is seen in the latent dynamics training for TA-ROM that the training loss trajectories are significantly different by several orders of magnitude,
which causes a large variance in extrapolation errors.
On the right, we report 
the extrapolation errors \eqref{eq:rel_l2_nonpara} by each method with respect to the wall time.
The means of the 10 simulations are shown as solid lines and the shaded areas correspond to one standard deviation away from the mean.
It is seen that Vanilla-FNN yields a rapid decay in the extrapolation error during the initial phase, however, it saturates quickly and progresses marginally as the training goes on. 
It can be seen that TA-ROM yields the largest variance in the extrapolation error.
This is because of their unstable behavior in the training.
Note that for TA-ROM, the test error is available only after the SAE training is complete. 
On the other hand, tLaSDI gives the smallest extrapolation errors for all 10 simulations and yields the smallest variance.
This indicates the robust prediction ability of tLaSDI in extrapolation.




 \begin{figure}[!ht]
	\centering
 {\includegraphics[width=0.49\textwidth]{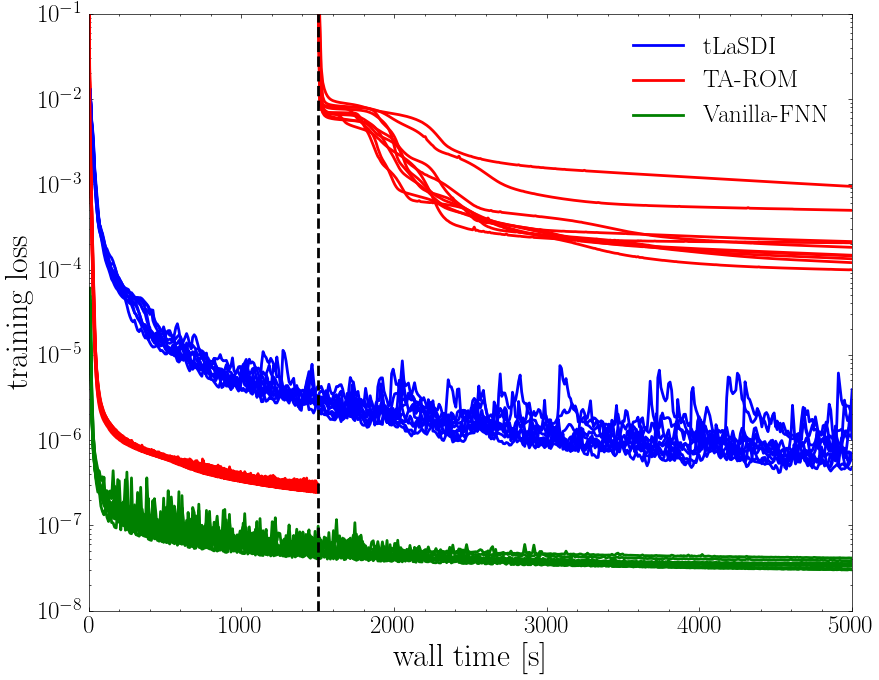}
 \includegraphics[width=0.49\textwidth]{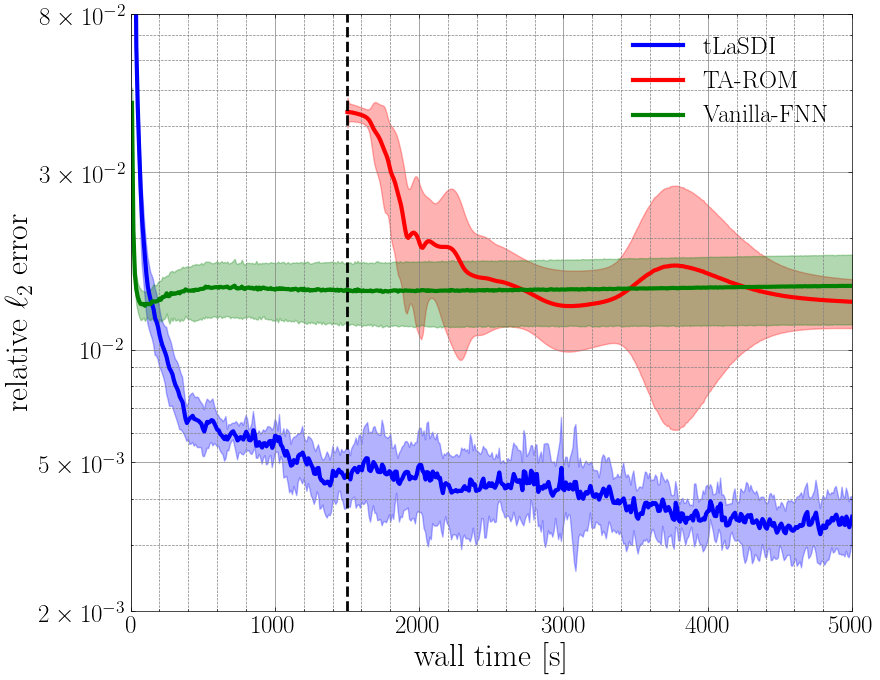}
 }
 \captionsetup{justification=justified}
 \caption{Example \ref{sec:oldroyd}. 
 Left: The training loss trajectories for 10 simulations versus the wall time by the three methods. Right: The mean and one standard deviation away from the mean of the extrapolation errors \eqref{eq:rel_l2_nonpara} versus the wall time. The vertical line indicates the time for TA-ROM trains SAE.}
 \label{fig:VC_loss_figures1}
 \end{figure}

In practice, the training loss is perhaps the only indicator for the users to decide whether the model is well-trained or not. Without additional information, one may choose the model that yields the smallest loss if multiple simulations are done.
In Figure~\ref{figure:var_4_nodes_tLaSDI_vs_TAROM_vs_VFNN}, we depict the prediction trajectories (extrapolation region) from the model with the smallest training loss among 10 simulations. 
The corresponding ground truth (GT) trajectories are shown in black dashed lines.
It is clearly observed that tLaSDI provides the most accurate extrapolation performance despite it being a non-intrusive ROM method. 
The extrapolation errors by tLaSDI, TA-ROM, and Vanilla-FNN are 
$2.47\times 10^{-3}$, $1.41\times 10^{-2}$ and $1.70\times 10^{-2}$, respectively.

 \begin{figure}[!ht]
	\centering
 {\includegraphics[width=\textwidth]{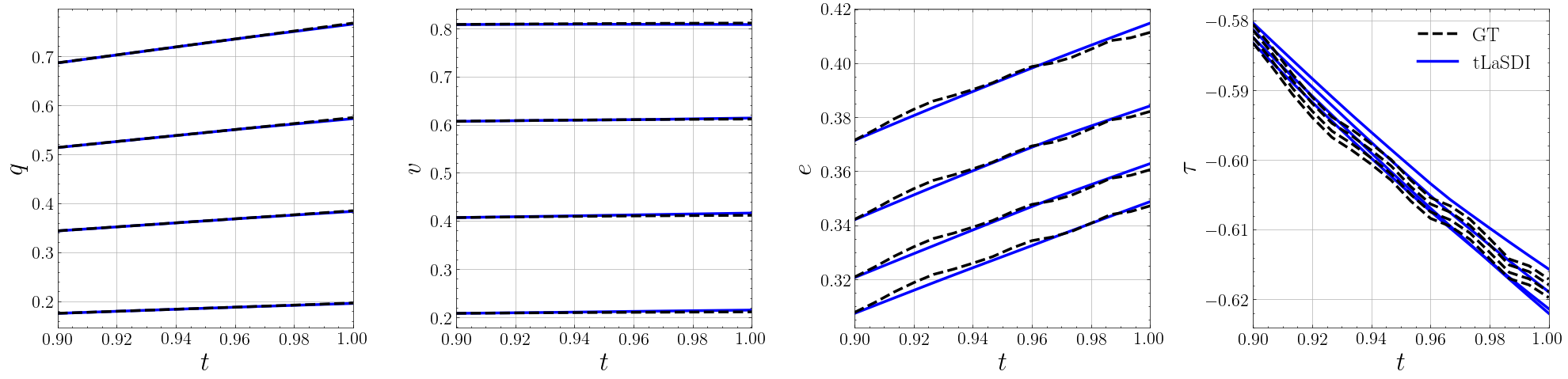}
 \includegraphics[width=\textwidth]{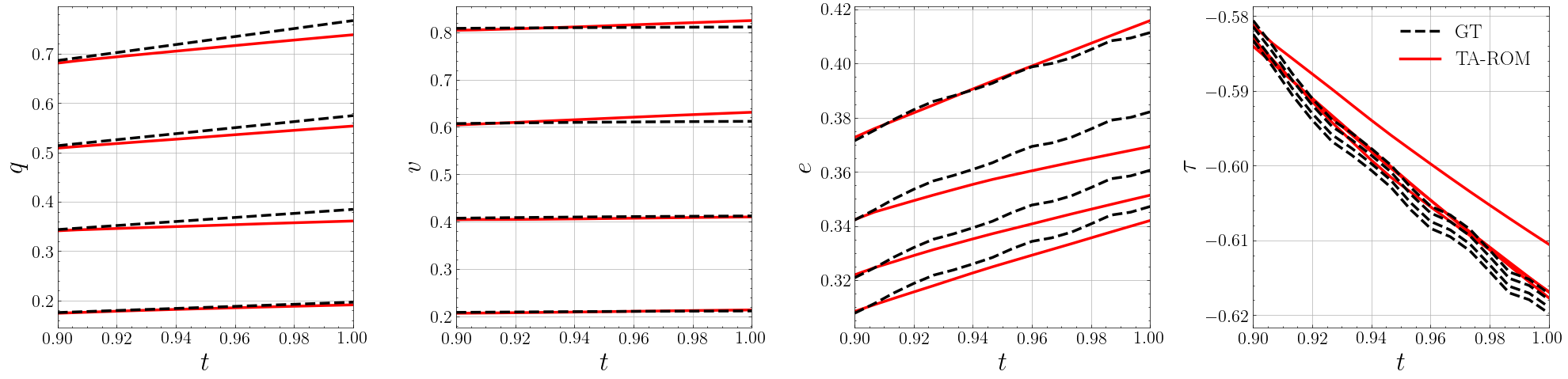}
  \includegraphics[width=\textwidth]{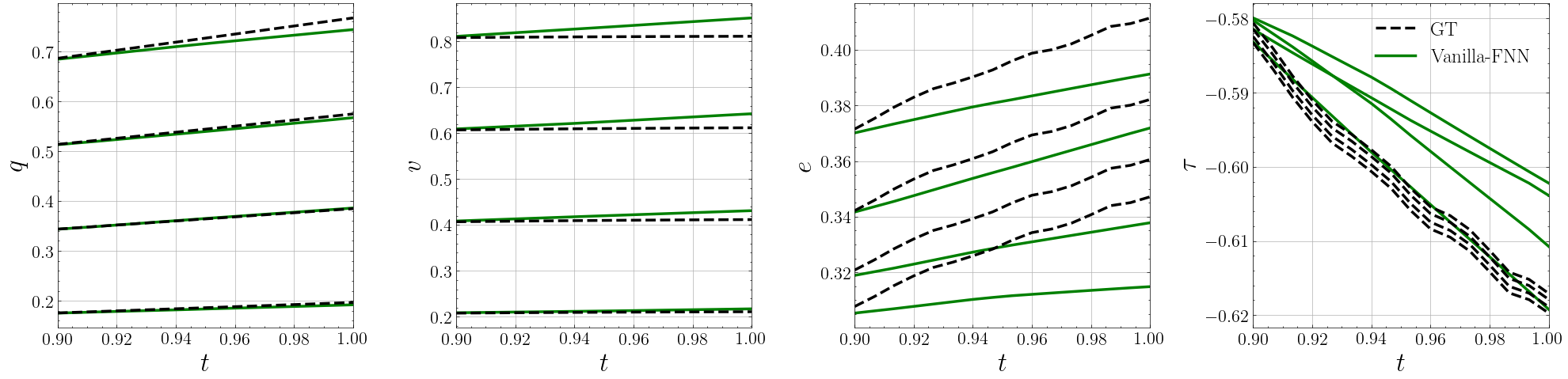}
 }
  \caption{Example \ref{sec:oldroyd}. Four different GT trajectories and the corresponding predictions by tLaSDI (top row), TA-ROM (middle row), and Vanilla-FNN (bottom row). Each method uses the model with the smallest loss from $10$ independent simulations.}
\label{figure:var_4_nodes_tLaSDI_vs_TAROM_vs_VFNN}
\end{figure}

We compare the effectiveness of the proposed loss formulation  \eqref{eq:loss_tLaSDI_GFINNs}.
In this regard, we consider four different loss configurations.
The first loss corresponds to the standard loss formulation where the first two terms ($\mathcal{L}_\text{int}, \mathcal{L}_\text{rec}$) of \eqref{eq:loss_tLaSDI_GFINNs} are used, which yields 
the extrapolation error of $1.19\times 10^{-2} \pm 2.57\times 10^{-3}$. 
The second loss is obtained by adding only $\mathcal{L}_\text{Jac}$ to the first one, which gives the error of $4.49\times 10^{-3} \pm 8.56\times 10^{-4}$.
The third loss is constructed from the first one by adding only the modeling loss $\mathcal{L}_\text{mod}$, which gives the error of $6.23\times 10^{-3} \pm 2.39\times 10^{-3}$.
The fourth is the one we propose \eqref{eq:loss_tLaSDI_GFINNs}, which gives the smallest extrapolation error of $3.57\times 10^{-3} \pm 4.81\times 10^{-4}$.
Altogether, it demonstrates the effectiveness of the proposed loss formulation.

\subsubsection{Two gas containers exchanging heat and volume}
\label{sec:gas_containers}
Two ideal gas containers, separated by a wall, are allowed to exchange heat and volume.
The dynamics of the wall are described by four variables: position ($q$) and momentum ($p$) of the moving wall, and the entropy of the gases in the two containers, ($S_1$, $S_2$). 
The evolution of the variables is then modeled by 
\begin{equation*}
\label{eq:GC}
\begin{pmatrix} \dot{q} \\ \dot{p} \\ \dot{S_1} \\\dot{S_2} \end{pmatrix}
 = \begin{pmatrix} p \\ \frac{2}{3}\left( \frac{E_1}{q}-\frac{E_2}{2-q}\right) \\ \frac{10}{T_1}\left( \frac{1}{T_1}-\frac{1}{T_2}\right) \\-\frac{10}{T_1}\left( \frac{1}{T_1}-\frac{1}{T_2}\right) \end{pmatrix}, 
 \quad 
 E_j = \frac{\exp\left(\frac{2S_j}{3}\right)}{(q+2(j-1)(1-q))^{2/3}},
\end{equation*}
where $T_j = \frac{\partial E_j}{\partial S_j}$,
$E_1$ and $E_2$ are the internal energy of the two containers given by \cite{schroeder1999introduction}.
This example is also considered in \cite{ottinger2005beyond,shang2020structure, zhang2022gfinns}.


Following \cite{hernandez2021deep}, we generate $100$ trajectories of different initial conditions that are randomly uniformly sampled from $[0.2,1.8]\times[-1,1]\times[1,3]\times[1,3]$.
If the variables for $i$-th trajectory are denoted by $q_i$, $p_i$, $S_{1,i}$, and $S_{2,i}$, we construct the full-state variable $\bfa{x}$ by concatenating them all, i.e., 
$\bfa{x}(t)  =  [\bfa{q}(t), \bfa{p}(t), \bfa{S}_1(t), \bfa{S}_2(t)]^\top \in \mathbb{R}^{400}$,
where $\bfa{q} = (q_i)$,
$\bfa{p} = (p_i)$,
$\bfa{S}_1 = (S_{1,i})$
and $\bfa{S}_2 = (S_{2,i})$.
Hence, the full-order dimension is $N=400$. 
We set $T=7.84$, $\delta=0.16$ with the fixed time step of $\Delta t = 0.02$.

In all the methods, we employ the autoencoders having $207K$ parameters with the latent space dimension of $n=30$.
The numbers of parameters for the latent space dynamics of Vanilla-FNN, TA-ROM, and tLaSDI are $910K$, $870K$, and $900K$ respectively.
The training of all models is terminated after about $12K$ seconds of wall-clock time, which amounts to approximately $85K$ and $960K$ iterations for tLaSDI and Vanilla-FNN respectively. 
For TA-ROM, the training involves $660K$ and $190K$ iterations for SAE and SPNN respectively, leading to a total of $850K$ iterations.
The number of iterations for SAE and SPNN is adjusted to enhance the performance of TA-ROM.
Other implementation details can be found in Appendix~\ref{sec:app_GC}.

Figure~\ref{fig:GC_loss_figures} shows the training loss and extrapolation error trajectories by the three methods with respect to the wall-clock time in seconds.
We run 10 independent simulations and report all the training loss trajectories on the left.
Similarly as before, we see a clear separation of the loss between the three methods because of the different scales of the training loss functions.
Yet, all the methods reach to some saturated levels at the end of the training.
On the right, we report the extrapolation errors \eqref{eq:rel_l2_nonpara} with respect to the wall time.
The means of the 10 simulations are shown as solid lines and the shaded areas are one standard deviation away from the mean.
Again, it is clearly observed that tLaSDI yields not only the smallest extrapolation error but also the smallest variance overall, especially toward the end of the training.
This again illustrates the effectiveness of tLaSDI in extrapolation.



\begin{figure}[!ht]
	\centering
 {\includegraphics[width=0.49\textwidth]{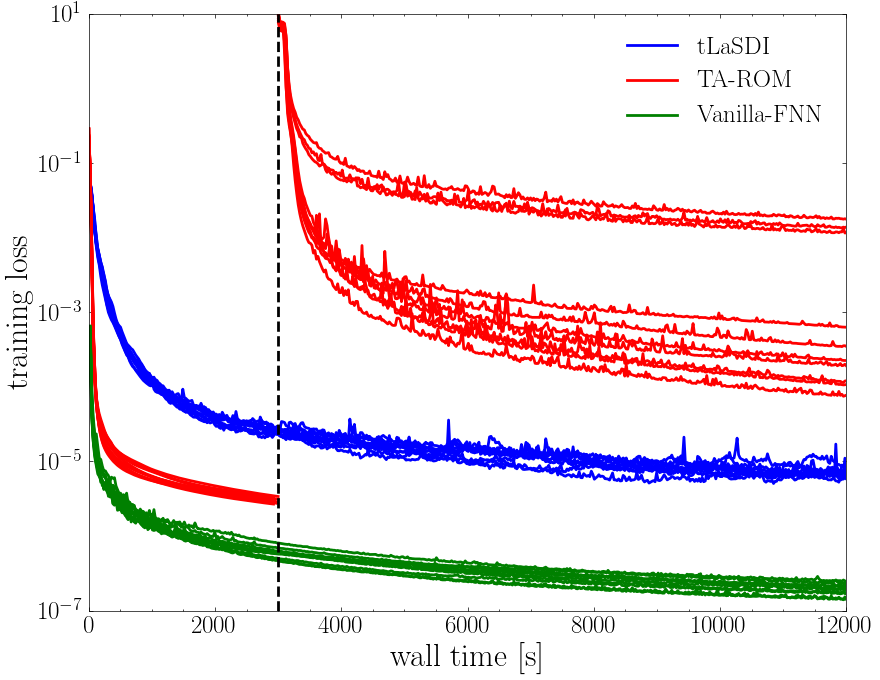}
 \includegraphics[width=0.49\textwidth]{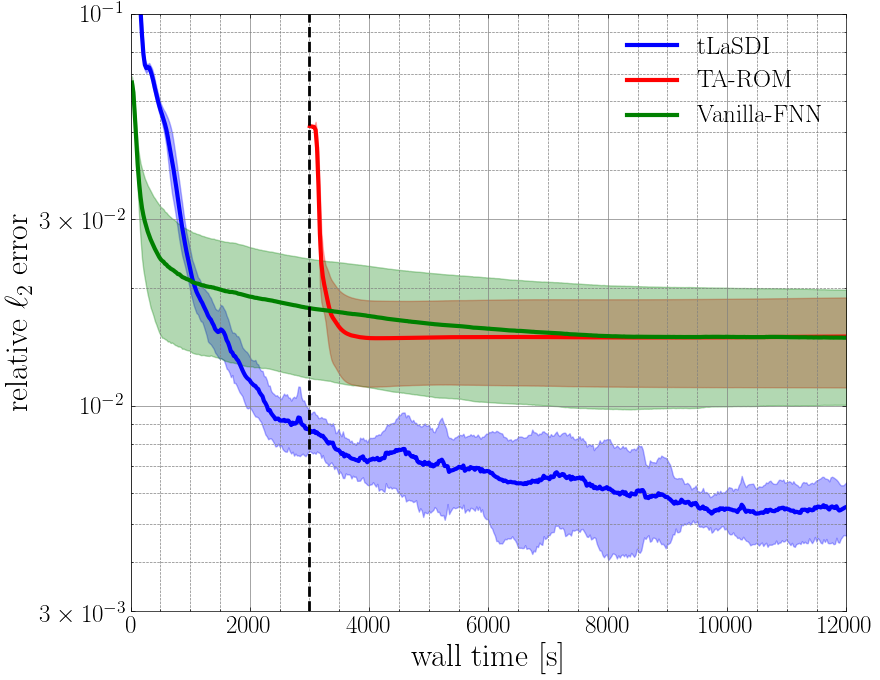}
 }
 \captionsetup{justification=justified}
 \caption{Example \ref{sec:gas_containers}. Left: The training loss trajectories for 10 simulations versus the wall time by the three methods. Right: The mean and one standard deviation away from the mean of the extrapolation errors \eqref{eq:rel_l2_nonpara} versus the wall time. The vertical line indicates the time for TA-ROM trains SAE.}
 \label{fig:GC_loss_figures}
 \end{figure}

In addition, we report some prediction trajectories (extrapolation regions) in 
Figure~\ref{figure:var_4_nodes_tLaSDI_vs_TAROM_vs_VFNN}. 
The trajectories are generated from the model that achieves the smallest loss among 10 independent simulations. 
The black dashed lines represent the corresponding GT trajectories. 
Again, it is clearly seen that tLaSDI exhibits robust extrapolation capability, especially for the position $q$ (first column) and the momentum $p$ (second column).
In contrast, TA-ROM and Vanilla-FNN yield incorrect trajectories for $q$ and $p$, which are typical behaviors of the non-intrusive approaches.
The extrapolation errors by tLaSDI, TA-ROM, and Vanilla-FNN are 
$3.96\times 10^{-3}$, $1.31\times 10^{-2}$ and $2.01\times 10^{-2}$, respectively.


 \begin{figure}[!ht]
	\centering
 {\centering
    \includegraphics[width=\textwidth]{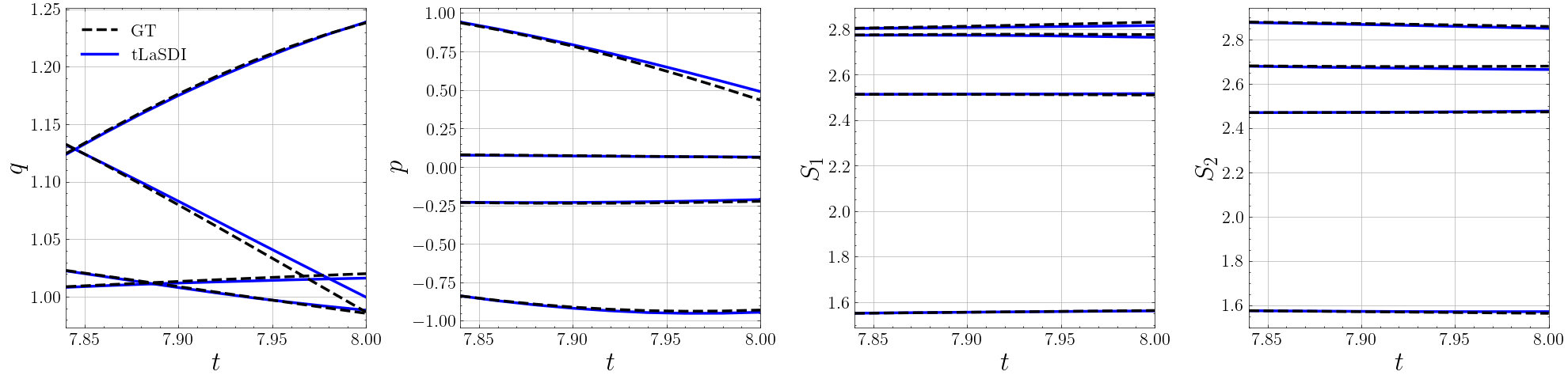}
  \includegraphics[width=\textwidth]{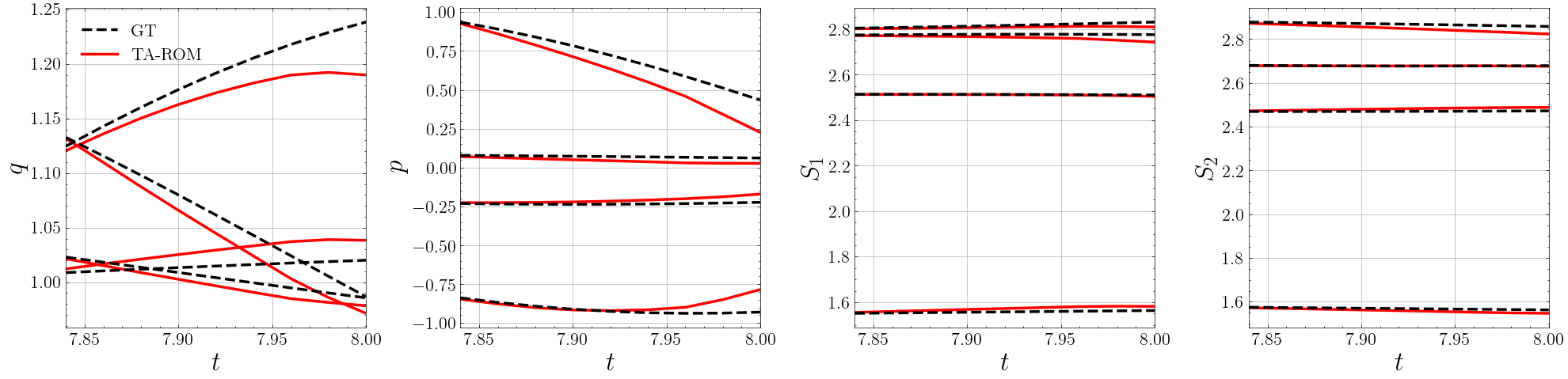}
  \includegraphics[width=\textwidth]{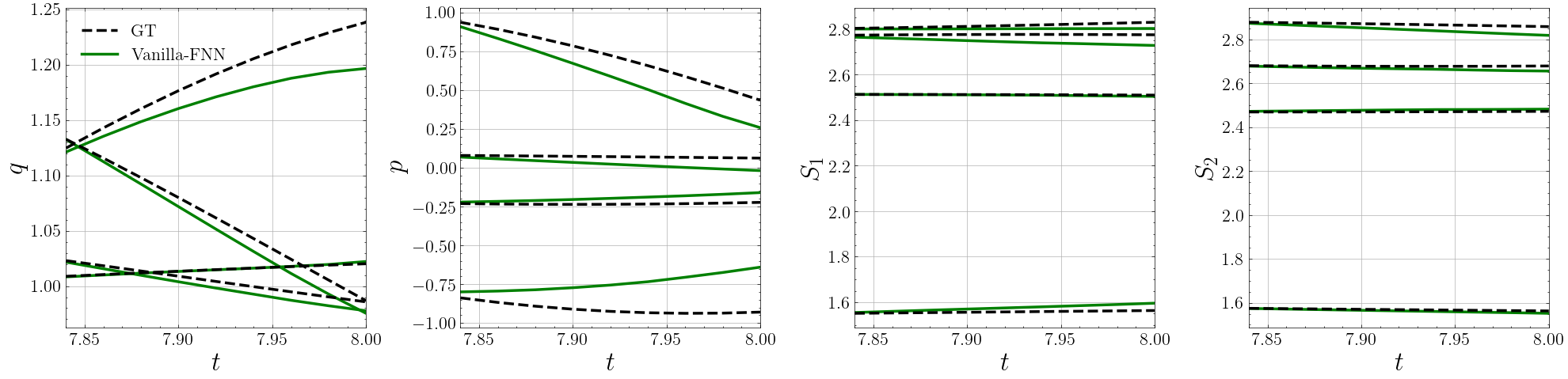}}
  \caption{Example \ref{sec:gas_containers}: Four different GT trajectories and the corresponding predictions by tLaSDI (top row), TA-ROM (middle row), and Vanilla-FNN (bottom row). Each method uses the model with the smallest loss from $10$ independent simulations.}
\label{figure:var_4_nodes_tLaSDI_vs_TAROM_vs_VFNN_GC}
\end{figure}

Lastly, we demonstrate the effectiveness of the proposed loss formulation \eqref{eq:loss_tLaSDI_GFINNs}.
Similar to the earlier example, we consider the four different loss configurations.
The standard loss yields the error of 
 $1.79 \times 10^{-2} \pm 6.05 \times 10^{-3}$ .
The modeling loss added to the standard one gives the error of 
$7.29 \times 10^{-3} \pm 8.26\times 10^{-4}$.
With the Jacobian loss being added to the standard loss, the error becomes  $5.52 \times 10^{-3} \pm 7.54\times 10^{-4}$.
The proposed loss with all the terms yields the error of $5.52 \times 10^{-3} \pm 8.13 \times 10^{-4}$.
In this example, we found that the Jacobian loss is particularly effective, which is newly introduced in the present work from Theorem~\ref{thm:total_err}.

\subsection{Parametric PDEs}
For parametric PDE examples, we are interested in capturing a distinct behavior of the solution over time, and at the same time, in achieving accurate predictions on unseen parameters. 
In this regard, the prediction error at a given parameter $\mu$ is measured by the maximum percentage relative $\ell_2$ error over time, i.e.,
\beq
\label{eq:max_rel_error}
e^{\text{max}\%}_{\mu} = 100 \times \max\limits_{k} \frac{\norm{\bfa{x}^k_\mu-\widetilde{\bfa{x}}^k_\mu}_{2}}{\norm{\bfa{x}^k_\mu}_{2}}. 
\eeq
Let $u_\mu(x,t)$ be the solution to a parametric PDE. 
The full state variable $\bfa{x}^k_\mu$ then corresponds to the collection of the solution at time $t_k$ evaluated at $N$ grid points $\{x_j\}_{j=1}^N$.
That is, $\bfa{x}_\mu^k = (u_\mu(x_j,t_k))_j \in \mathbb{R}^N$.

To ensure high prediction accuracy, the training data $\bfa{x}_\mu^k$ need to be collected at carefully selected parameters, which are often referred to as training parameters or training samples.
In this regard, we employ the greedy sampling algorithm proposed in \cite{he2023glasdi} for training data generation.
The greedy sampling utilizes a physics-informed residual-based error indicator to select proper training parameters from $\Omega_\mu$, during training.
Training data is produced on-the-fly at the sampled parameters during the training process.
The greedy sampling is known to outperform the traditional predetermined uniform (equidistant) sampling in terms of prediction accuracy \cite{he2023glasdi}.
We remark that the greedy sampling method may yield distinct training parameter selections for different ROM models. 
The implementation details of greedy sampling for the following examples can be found in Appendix \ref{sec:app_1DBG}.

We compare tLaSDI with another NM ROM, gLaSDI \cite{he2023glasdi} that proposed the aforementioned greedy sampling algorithm. 
gLaSDI uses a user-defined library of candidate basis functions (e.g., polynomials, trigonometric functions) to represent the latent space dynamics. 
A brief overview of gLaSDI can be found in Appendix \ref{app:overview-other-method}.


\subsubsection{1D Burgers' equation}\label{sec:1DBG}
We consider the 1D parametric inviscid Burgers' equation, which serves as a simplified model to exhibit the formation of shock waves in fluid dynamics and nonlinear acoustics. 
The parametric PDE under consideration reads
\beq
\label{eq:1DBG}
\begin{split}
\frac{\partial u_{\mu}}{\partial t} + u \frac{\partial u_{\mu}}{\partial x} = 0, \quad &\forall \ (x,t) \in (-3,3) \times (0,2], \\
u_{\mu}(3,t) = u_{\mu}(-3,t), \quad &\forall t \in (0,2] \\
u_{\mu}(x,0) = \alpha \exp\left(-\frac{x^2}{2\omega^2}\right), \quad &\forall x \in (-3,3),
\end{split}
\eeq
where the initial condition is parametrized by the amplitude $\alpha$, and the width $\omega$.
Let $\mu = (\alpha, \omega)$ and the parameter domain be $\Omega_{\mu}= [0.7,0.8]\times[0.9,1.0]$. 
Both parameters $\omega$ and $\alpha$ are discretized into $21$ evenly spaced points within their respective ranges. This discretization yields a set of $441$ distinct parameters in the parameter domain.
The training data are discretized with the spatial and time spacing $\Delta x = 6/200$, $\Delta t = 2/200$. Consequently, the dimension of the full-order model is $N=201$. 
We collect $25$ training data at the parameters selected using the greedy sampling algorithm, during training. 
At the end of the training, tLaSDI and gLaSDI generated distinct sets of training data. 
Additional details on implementation are provided in Appendix \ref{sec:app_1DBG}.

The latent dimension is set to $n=10$. The total number of parameters for hyper-autoencoder and GFINNs in tLaSDI are approximately $905K$ and $35K$ respectively.
tLaSDI is trained by \texttt{Adam} optimizer \cite{kingma2014adam} for $42K$ iterations.

The heatmaps of the maximum percentage relative ${\ell}_2$ errors (\ref{eq:max_rel_error}) of tLaSDI and gLaSDI computed at all 441 parameters are depicted in Figure \ref{fig:BG_tLaSDI_vs_gLaSDI}. 
The black square boxes indicate the training parameters sampled using the greedy sampling algorithm. 
It is observed that the sampled parameters for tLaSDI tend to cluster near the boundaries, while no notable pattern is observed for gLaSDI.
In terms of accuracy, tLaSDI and gLaSDI give the worst-case errors of $1.5\%$ and $6.0\%$, respectively, among $441$ parameters.


tLaSDI constructs a thermodynamic structure in the latent dynamics which provides the NN-based entropy function $S_{\text{NN}}$ in the latent space.
Figure~\ref{fig:BG_tLaSDI_vs_gLaSDI} shows the data-driven entropy function and its rate of change with respect to time.
The mean over all the test parameters is reported and the shared area is one standard deviation away from the mean. 
As promised by GFINNs, we see that the entropy is increasing as the rate is always non-negative. 
The entropy production rate $\frac{d}{dt}S_{\text{NN}}(\bfa{z}(t))$ has the largest value at the final time $t=2$, which happens to be the time when the full-state solution exhibits the stiffest pattern.
See the right of Figure~\ref{fig:BG_tLaSDI_vs_gLaSDI} for the tLaSDI prediction at varying times.

  \begin{figure}[t]
	\centering
	\captionsetup{justification=centering}
		\begin{subfigure}{0.49\textwidth}
  \includegraphics[width=\textwidth]{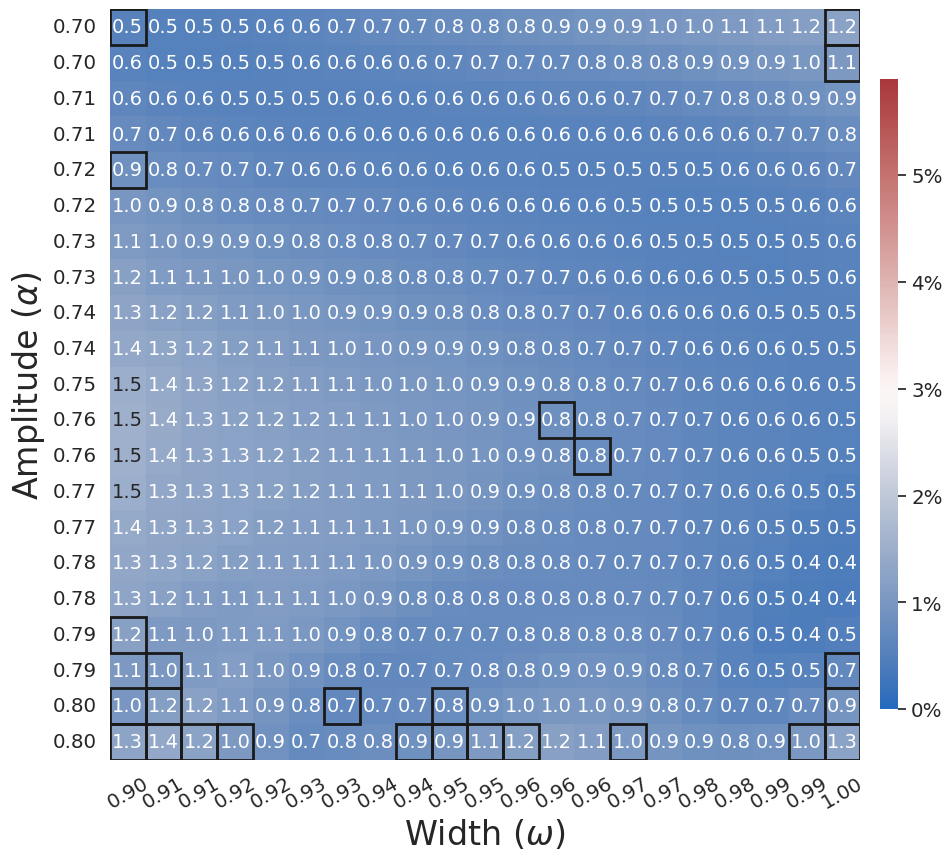}
  \caption{tLaSDI}
  \label{fig:1DBG_heatmap_gLaSDI_latent10}
\end{subfigure}
  \begin{subfigure}{0.49\textwidth}
  \includegraphics[width=\textwidth]{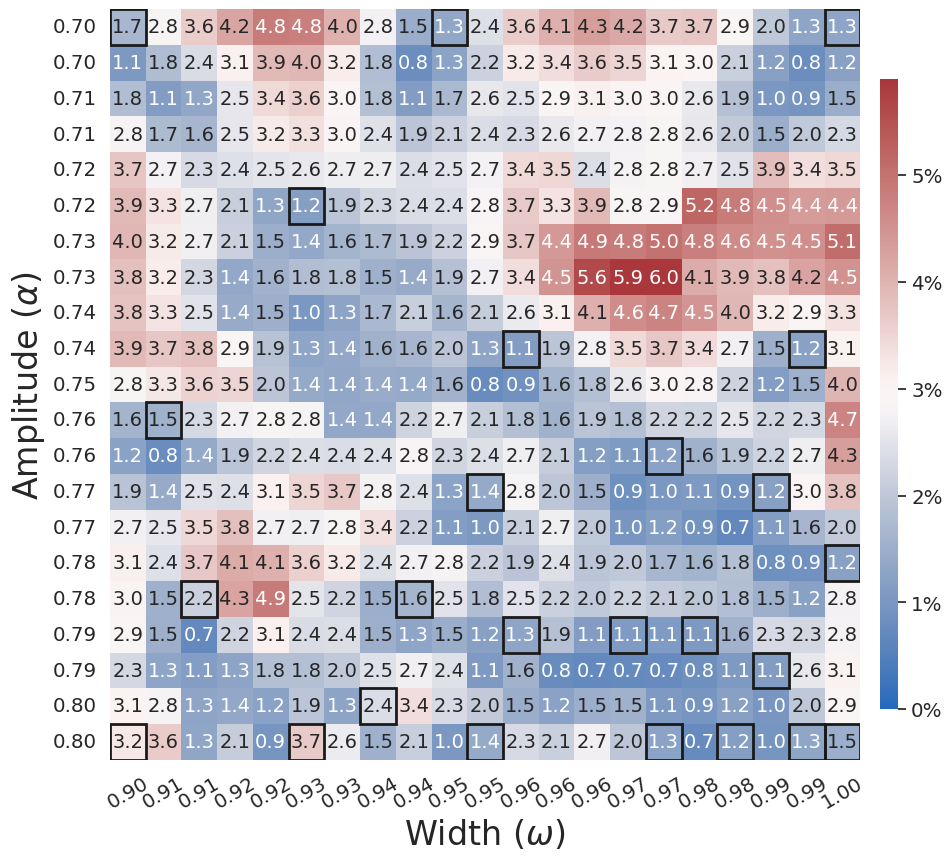}
  \caption{gLaSDI}
  \label{fig:1DBG_heatmap_tLaSDI_latent10}
 \end{subfigure}
 \captionsetup{justification=justified}
 \caption{Example \ref{sec:1DBG}. Comparison between tLaSDI and gLaSDI \cite{he2023glasdi}. The number on each box represents the maximum percentage relative error (\ref{eq:max_rel_error}) computed at each parameter. The black square boxes denote the positions of the sampled training points.}
 \label{fig:BG_tLaSDI_vs_gLaSDI}
 \end{figure}

     \begin{figure}[!th]
	\centering
 {\includegraphics[width=0.32\textwidth, height=3.25cm]{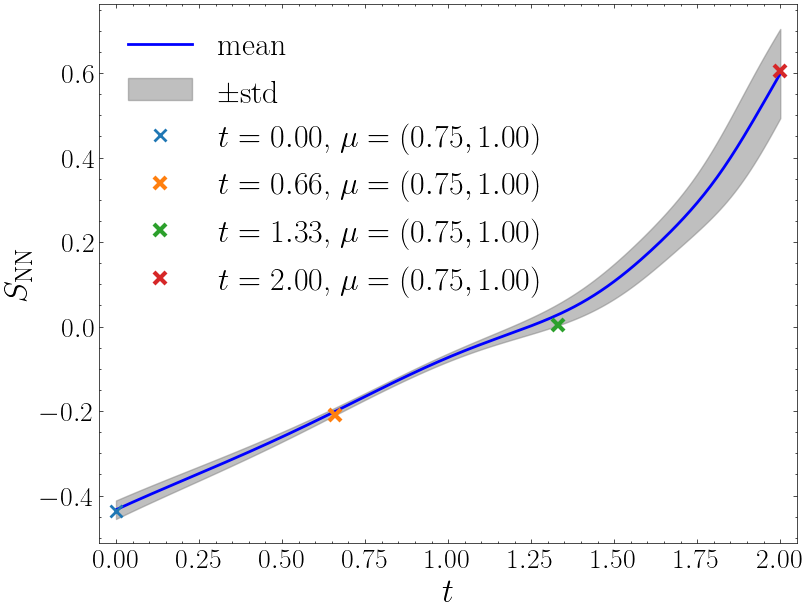}
 \includegraphics[width=0.32\textwidth, height=3.25cm]{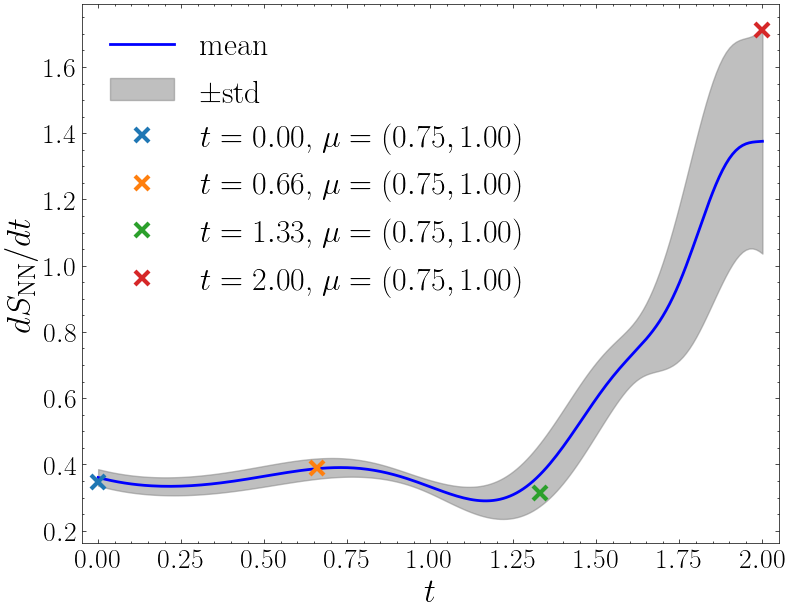}
 \includegraphics[width=0.32\textwidth, height=3.25cm]{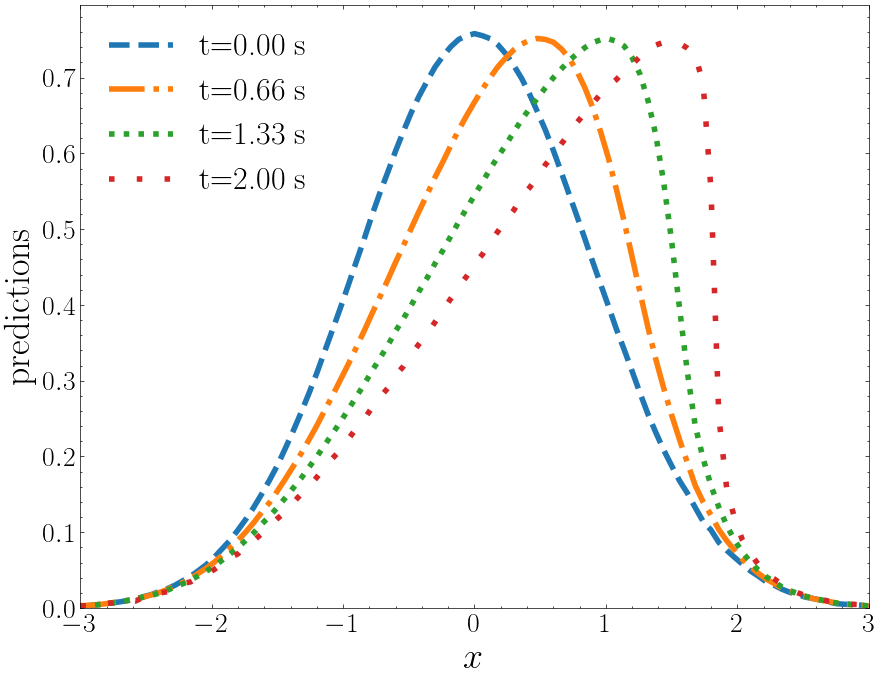}
 }
 \captionsetup{justification=justified}
 \caption{Example \ref{sec:1DBG}. The mean and one standard deviation away from the mean of (left) the entropy $S_\text{NN}$ and (middle) $\frac{d}{dt}S_\text{NN}$ across the test parameters. Right: The solution prediction by tLaSDI at the four times at $\mu = (0.75,1.00)$ whose corresponding $S_\text{NN}$ and  $\frac{d}{dt}S_\text{NN}$ are marked in the left and middle figures.}
 \label{fig:BG_mean_std_S_dSdt_sol}
 \end{figure}

\subsubsection{1D heat equation}
\label{sec:1DHT}
To further investigate the correlation captured by the data-driven entropy from tLaSDI, we consider the following parametric heat equation:
\begin{equation*}
\begin{split}
\frac{\partial u_{\mu}}{\partial t} - \frac{\partial^2 u_{\mu}}{\partial x^2}= 0, \quad &\forall \ (x,t) \in (-3,3) \times (0,2],
\end{split}
\end{equation*}
with the same boundary and initial conditions to the Burgers' equation (\ref{eq:1DBG}).
We follow the same setup as the earlier example and report the details in Appendix~\ref{sec:app_1DBG}.

In Figure~\ref{fig:HT_mean_std_S_dSdt_sol}, we report the graphs of the data-driven entropy from tLaSDI
and its rate along with the tLaSDI prediction at varying times. 
In contrast with the Burgers' equation, it is clearly observed that the entropy production rate $\frac{d}{dt}S_{\text{NN}}(\bfa{z}(t))$ decreases over time, reflecting the diffusive behavior of the solution to the heat equation.
This comparison illustrates a distinct feature of tLaSDI through the data-driven entropy which exhibits a strong correlation with the underlying physical processes.
     \begin{figure}[h]
	\centering
 {\includegraphics[width=0.32\textwidth, height=3.25cm]{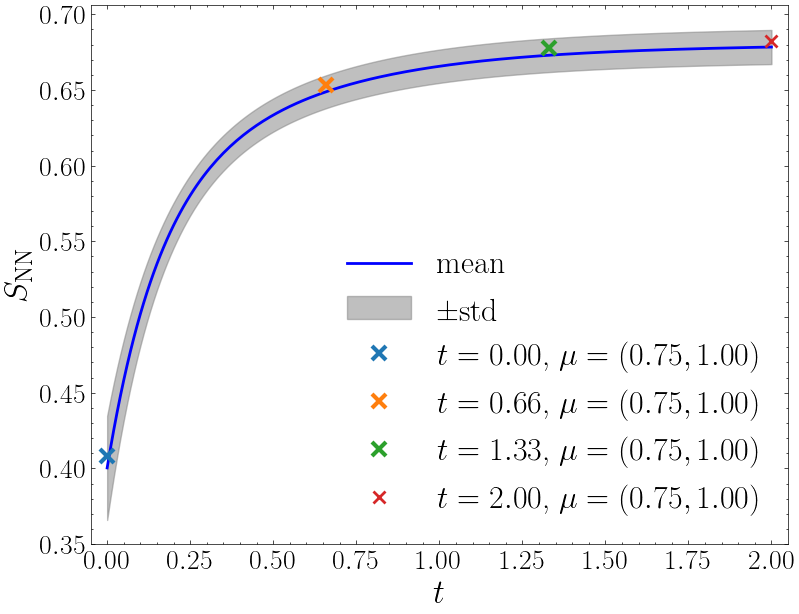}
 \includegraphics[width=0.32\textwidth, height=3.25cm]{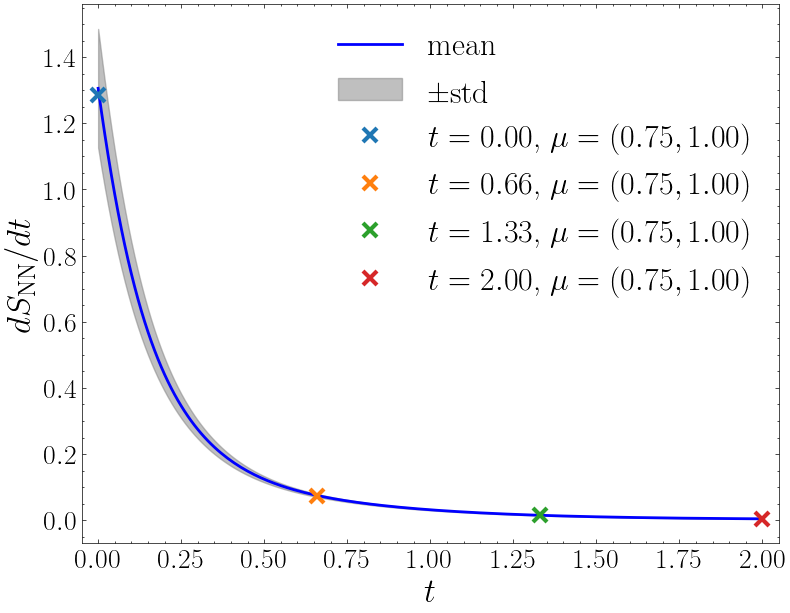}
 \includegraphics[width=0.32\textwidth, height=3.25cm]{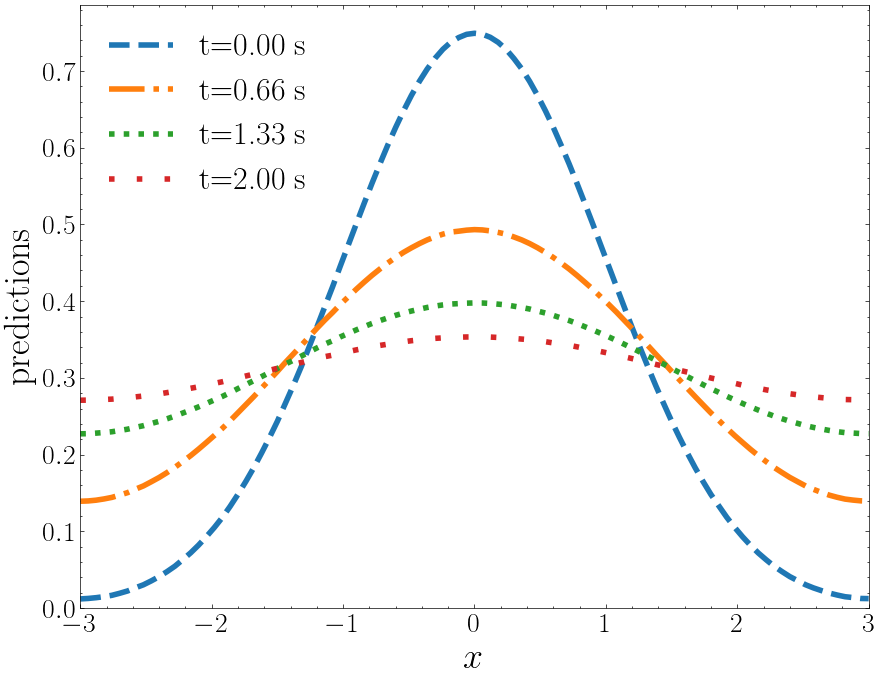}
 }
 
 \captionsetup{justification=justified}
 \caption{Example \ref{sec:1DHT}. The mean and one standard deviation away from the mean of (left) the entropy $S_\text{NN}$ and (middle) $\frac{d}{dt}S_\text{NN}$ across the test parameters. Right: The solution prediction by tLaSDI at the four times at $\mu = (0.75,1.00)$ whose corresponding $S_\text{NN}$ and  $\frac{d}{dt}S_\text{NN}$ are marked in the left and middle figures.}
 \label{fig:HT_mean_std_S_dSdt_sol}
 \end{figure}

\section{Conclusions}
\label{sec:conclusion}
We propose a non-intrusive thermodynamics-informed ROM method, namely, tLaSDI.
tLaSDI uses an autoencoder to construct a nonlinear manifold as the latent space and models the latent dynamics through GFINNs.
GFINNs are structured to precisely satisfy the first and second laws of thermodynamics through the GENERIC formalism.
Based on an abstract error estimate of the ROM approximation, a new loss is formulated, which significantly improves the performance of tLaSDI.
Numerical examples are presented to demonstrate the effectiveness of tLaSDI.
For the data-driven discovery tasks, tLaSDI exhibits robust extrapolation ability, which typical pure data-driven approaches lack.
For the parametric PDE problems, we found that tLaSDI provides smaller prediction errors on unseen parameters when it is compared with gLaSDI.
Due to the thermodynamic infusion, tLaSDI returns the data-driven entropy function defined in the latent space.
We found that this entropy function from tLaSDI captures not only the time on which the solution to the Burgers' equation is the stiffest but also the time on which the solution to the heat equation is the most diffusive through the entropy production rate.

\backmatter
\clearpage




\bmhead{Acknowledgements}
J.\ S.\ R.\ Park was partially supported by Miwon Du-Myeong Fellowship via Miwon Commercial Co., Ltd. and  a KIAS Individual Grant (AP095601) via the Center for AI and Natural Sciences at Korea Institute for Advanced Study.
J.\ S.\ R.\ Park would like to thank Dr. Quercus Hernandez and Zhen Zhang for their helpful guidance on the implementation of TA-ROM and GFINNs. 
S.\ W.\ Cheung and Y.\ Choi were partially supported for this work by Laboratory Directed Research and Development (LDRD) Program by the U.S. Department of Energy (24-ERD-035). 
Y.\ Choi was partially supported for this work by the U.S. Department of Energy, Office of Science, Office of Advanced Scientific Computing Research, as part of the CHaRMNET Mathematical Multifaceted Integrated Capability Center (MMICC) program, under Award Number DE-SC0023164 at Lawrence Livermore National Laboratory. 
Y.\ Shin was partially supported for this work by the NRF grant funded by the Ministry of Science and ICT of Korea (RS-2023-00219980). 
Lawrence Livermore National Laboratory is operated by Lawrence Livermore National Security, LLC, for the U.S. Department of Energy, National Nuclear Security Administration under Contract DE-AC52-07NA27344. IM release number: LLNL-JRNL-860848.

\noindent






\clearpage
\begin{appendices}

\section{An overview of existing methods} \label{app:overview-other-method}
In this section, we provide a brief overview of the existing NM ROM methods. 

{\bf Vanilla-FNN}: 
This algorithm simply leverages the standard feed-forward neural network (FNN) for both autoencoder and latent space dynamics. 
It employs a loss function composed of two standard components, integration loss ${\mathcal L}_{\text{int}}$ and reconstruction loss ${\mathcal L}_{\text{rec}}$.
In this framework, the autoencoder and the latent space dynamics (FNN) are trained simultaneously.

{\bf TA-ROM} (Thermodynamics-aware reduced-order models \cite{hernandez2021deep}): In this algorithm, the sparse autoencoder (SAE) \cite{ng2011sparse} is utilized for the dimension reduction.
The SPNN \cite{hernandez2021structure} is employed to embed the GENERIC formalism into latent space dynamics. 
In this framework, they train SAE first, and then SPNN subsequently. 
The training procedure utilizes the following loss formulation
\begin{equation*}
\label{eq:loss_tarom}
\bsp
{\mathcal L}_{\text{SAE}} & =  {\mathcal L}_{\text{rec}} + \lambda_{\text{spr}} {\mathcal L}_{\text{spr}}, \\
{\mathcal L}_{\text{SPNN}} & = \lambda_{\text{int}}{\mathcal L}_{\text{int}} +  {\mathcal L}_{\text{deg}}, 
\end{split}
\end{equation*}
in addition to the standard $\ell_2$ regularization for the weight decay of SPNN. In this framework, the dimension of latent space is not predetermined. It is rather discovered by minimizing
the sparsity loss ${\mathcal L}_{\text{spr}}$ \cite{ng2011sparse,hernandez2021deep} during SAE training to eliminate insignificant components of the latent vector.
The degeneracy condition of GENERIC formalism is not exactly satisfied with SPNN. Instead, it is enforced through the minimization of the degeneracy loss component ${\mathcal L}_{\text{deg}}$:
\begin{equation*}
\label{eq:loss_degeneracy}
{\mathcal L}_{\text{deg}} = \sum\limits_{k} \left(\norm{L_{NN} \nabla S_{NN} }_2^2 + \norm{M_{NN}\nabla E_{NN} }_2^2 \right),
\end{equation*}
where the neural networks are evaluated at $\phi_{\text{e}}(\bfa{x}^{k})$.  

{\bf gLaSDI} (Parametric physics-informed greedy latent space dynamics identification \cite{he2023glasdi}): This algorithm is developed for ROM of high-dimensional parametric dynamical systems. 
Its primary distinction from our proposed method lies in the latent space dynamics identification (DI) model. In gLaSDI, the latent space dynamics $F^r_{\mu}$ are identified using a user-defined library of candidate basis functions (e.g., polynomial, trigonometric, exponential functions) at each training parameter.  
To enable the identification of latent dynamics at unseen parameters, the $k$-nearest neighbors (KNN) convex interpolation is utilized.
The algorithm uses a greedy sampling strategy based on a physics-informed residual-based error indicator to select training parameters during the training process. 
The goal of greedy sampling is to efficiently generate training data on-the-fly that yield optimal prediction performance. 
In this framework, the autoencoder and the DI model are trained simultaneously minimizing the loss function ${\mathcal L}_{\text{gLaSDI}} = {\mathcal L}_{\text{rec}}+\lambda_{\text{mod}}{\mathcal L}_{\text{mod}}$.

The comparison of tLaSDI with TA-ROM and Vanilla-FNN is introduced in  Table \ref{tab:diff_tLaSDI_TA_ROM_gLaSDI}.
While specific DI models are proposed by default within each framework, it is worth noting that tLaSDI and TA-ROM are compatible with both GFINNs and SPNN.
However, for simplicity, the numerical results presented in this work adhere to the conventional pairings: tLaSDI with GFINNs and TA-ROM with SPNN.

 \begin{table}[ht!]
\centering
\begin{tabular}{|c||c|c|c|}
\hline
Method & DI model  & Loss components & Training AE/DI model\\
\hline
tLaSDI & GFINNs & ${\mathcal L}_{\text{int}}$, ${\mathcal L}_{\text{rec}}$, ${\mathcal L}_{\text{Jac}}$, ${\mathcal L}_{\text{mod}}$ & Simultaneous\\ 
 \hline
TA-ROM & SPNN & ${\mathcal L}_{\text{int}}$, ${\mathcal L}_{\text{rec}}$, ${\mathcal L}_{\text{spr}}$, ${\mathcal L}_{\text{deg}}$ & Separate\\
\hline
 Vanilla & FNN & ${\mathcal L}_{\text{int}}$, ${\mathcal L}_{\text{rec}}$ & Simultaneous\\ 
\hline
gLaSDI & Dictionary & ${\mathcal L}_{\text{rec}}$, ${\mathcal L}_{\text{mod}}$ & Simultaneous\\
\hline
\end{tabular}
\caption{Comparative overview of tLaSDI and other methods}
\label{tab:diff_tLaSDI_TA_ROM_gLaSDI}
\end{table}

\section{Implementation details}
\label{appendix:training_details}
In this appendix, we outline the implementation details, including training data generation, hyperparameter settings and NN architectures for the ROM models implemented in all numerical examples in Section \ref{sec:num}. 
\subsection{Couette flow of an Oldroyd-B fluid}
\label{sec:app_VC}
This section is about the implementations of tLaSDI, TA-ROM and Vanilla-FNN for Section \ref{sec:oldroyd}. 
The TA-ROM implementation is guided by the training details presented in \cite{hernandez2021deep} and the associated source code\footnote{\label{tarom}GitHub page: \url{https://github.com/quercushernandez/DeepLearningMOR}}.
The training of all methods is based on a full batch of training data.
In addition to minimizing the loss function, the standard ${\ell}_2$ regularization is utilized for the weight decay of DI models of each method. 
Learning rate schedulers are used for all methods. 
The Runge-Kutta-Fehlberg method (RK45) is utilized for the time-step integration of latent space dynamics during the training and prediction phases of all models. 

{\bf Data generation}: The training database is provided by the authors of \cite{hernandez2021deep} through the GitHub page\footnotemark[2]. 
Derivative data are derived from the training data using a central difference approximation. 
For endpoints at $t=0$ and $t=T$, where the central difference is inapplicable, forward and backward difference schemes are used, respectively.

{\bf Hyperparameters and NN architectures}:
The hyperparameters for each method are selected in order to optimize the predictive performance.
Table \ref{tab:loss_weights_all_methods_VC} lists hyperparameters employed for each method including loss weights, standard ${\ell}_2$ regularization weight ($\lambda_{\text{reg}}$),
initial learning rates ($l_r$) and their decay rates. 
The integration loss weight for tLaSDI and Vanilla-FNN is fixed at $1$. The reconstruction loss weight for tLaSDI and Vanilla-FNN is chosen from $\{10^{-5}, 10^{-4}, \dots, 1\}$.
Then the weight for the Jacobian loss component of tLaSDI is selected from $\{10^{-5}, 10^{-4}, \dots, 1\}$. Finally, we choose the model loss weight for tLaSDI from $\{10^{-8}, 10^{-7}, \dots, 1\}$.
The reconstruction and degeneracy loss weights for TA-ROM are set to $1$.
The sparsity and integration loss weights for TA-ROM are selected from $\{10^{-6}, 10^{-5}, \dots, 10^{-2}\}$ and $\{1, 10, \dots, 10^{4}\}$, respectively.
The initial learning rates for all models are selected from $\{10^{-6}, 10^{-5}, 10^{-4} , 10^{-3}\}$.
The learning rates of an autoencoder and DI model for tLaSDI and Vanilla-FNN decrease after every $1K$ iterations until they reach a minimum of $10^{-5}$.
For TA-ROM, the learning rate of SPNN is reduced after every $1K$ iterations until it reaches $10^{-6}$.

For the encoders of all models, we use 400-160-160-8 FNN architectures with ReLU activation functions, and symmetric architectures are employed for the decoders.
All NNs within the DI models of tLaSDI and TA-ROM have $5$ layers and $100$ neurons in each hidden layer, with hyperbolic tangent activation functions. 
The FNN for the latent space dynamics of Vanilla-FNN has $5$ layers and $215$ neurons within each hidden layer, with a hyperbolic tangent activation function.
All models consist of the same level of number of NN parameters given their architectures.


In TA-ROM, special attention should be given to adjusting the sparsity loss weight and the duration of SAE training due to its separate training nature.
It is observed that certain sparsity loss weights (e.g., $10^{-4}$) and durations of SAE training often yield overfitting or large standard deviation of the extrapolation errors. 
In our implementations, the sparsity loss weight $10^{-6}$ worked the best with $1500$ seconds of SAE training.

\begin{table}[h]
\centering
\begin{tabular}{|c||c|c|c|c|c|c|c|c|c|}
\hline
Method&$\lambda_{\text{int}}$& $\lambda_{\text{rec}}$ & $\lambda_{\text{Jac}}$& $\lambda_{\text{mod}}$  &$\lambda_{\text{deg}}$ &$\lambda_{\text{spr}}$ &$\lambda_{\text{reg}}$  & Initial $l_r$ & $l_r$ decay rate  \\
\hline
tLaSDI  & 1 &  1e-1 & 1e-2  &1e-8 & N/A & N/A & 1e-8& 1e-4& 1\%\\ 
 \hline
TA-ROM  & 1e2 &  1  & N/A  & N/A  & 1 & 1e-6& 1e-5& 1e-5/1e-5 & 0\% / 5\% \\ 
\hline
Vanilla-FNN & 1 &  1e-1  & N/A  & N/A  & N/A & N/A&1e-8& 1e-4 &  1\%\\ 
\hline
\end{tabular}
\caption{Example \ref{sec:oldroyd}: Hyperparameters used for implementations of tLaSDI, TA-ROM, and Vanilla-FNN; The learning rate and its decay rate for TA-ROM correspond to SAE/SPNN.}
\label{tab:loss_weights_all_methods_VC}
\end{table}

\subsection{Two gas containers exchanging heat and volume}
\label{sec:app_GC}
The implementation details of tLaSDI, TA-ROM, and Vanilla-FNN, for the results presented in Section \ref{sec:gas_containers}, are presented in this section.
The TA-ROM implementation is guided by the training details in \cite{hernandez2021deep} and its associated code\footnote{GitHub page: \url{https://github.com/quercushernandez/DeepLearningMOR}}.
Full batches of training data are used in the training of all methods. 
Learning rate schedulers are utilized for training all models. 
The Runge-Kutta second/third order integrator (RK23) is employed for temporal integration of latent space dynamics for all methods. 
The standard ${\ell}_2$ regularization is used for the weight decay of SPNN for TA-ROM. 

{\bf Data generation}: The training data are generated by the Runge-Kutta-Fehlberg method (RK45) based on the source code\footnote{GitHub page: \url{https://github.com/zzhang222/gfinn\_gc/tree/main}} from the original GFINNs study \cite{zhang2022gfinns}. 
The corresponding derivative data are obtained by applying the central difference approximation scheme to the training data.
At the endpoints $t=0$ and $t=T$, where the central difference scheme is not applicable, we employ forward and backward difference schemes, respectively.

{\bf Hyperparameters and NN architectures}: The hyperparameters for all methods are adjusted to enhance the extrapolation performance. 
Table \ref{tab:loss_weights_all_methods_GC} presents the hyperparameter settings for each method including loss weights, standard ${\ell}_2$ regularization weights ($\lambda_{\text{reg}}$),
initial learning rates ($l_r$), and the rate of learning rate decay.
For tLaSDI and Vanilla-FNN, the integration loss weight is consistently set at $1$. Selection of the reconstruction loss weight for both tLaSDI and Vanilla-FNN ranges within $\{10^{-5}, 10^{-4}, \dots, 1\}$. Then the Jacobian loss weight for tLaSDI is determined from the set $\{10^{-5}, 10^{-4}, \dots, 1\}$. 
Finally, the model loss weight is chosen from $\{10^{-8}, 10^{-7}, \dots, 1\}$. 
For TA-ROM, both reconstruction and degeneracy loss weights are fixed at $1$. 
The choices for sparsity and integration loss weights for TA-ROM are made from $\{10^{-6}, 10^{-5}, \dots, 10^{-2}\}$ and $\{1, 10, \dots, 10^{4}\}$, respectively.
Initial learning rates for all models are selected from $\{10^{-6}, 10^{-5}, 10^{-4}, 10^{-3}\}$. 
The learning rates for autoencoders and DI models for all models are reduced after every $1K$ iterations until they reach $10\%$ of the initial learning rates.

We use 400-200-100-30 FNN architectures for the encoders of all models and symmetric architectures for the decoders, with ReLU activation functions, 
All NNs within the DI models of tLaSDI and TA-ROM employ $5$ layers and $200$ neurons in each hidden layer, with sine activation functions. 
The latent space dynamics (FNN) of Vanilla-FNN consists of $5$ layers and $540$ neurons within each hidden layer, with a sine activation function.
Given the architectures, all models consist of the same level of number of trainable NN parameters. 

For TA-ROM, which utilizes the separate training of SAE and SPNN, the training duration of SAE should be carefully selected for optimal predictive performance. 
In our implementations, the duration of SAE training affects the reconstruction error of SAE and standard deviations of extrapolation errors. 
The best extrapolation performance was achieved with $3K$ seconds of SAE training. 

\begin{table}[h]
\centering
\begin{tabular}{|c||c|c|c|c|c|c|c|c|c|}
\hline
Method &$\lambda_{\text{int}}$& $\lambda_{\text{rec}}$ & $\lambda_{\text{Jac}}$& $\lambda_{\text{mod}}$ &$\lambda_{\text{deg}}$ &$\lambda_{\text{spr}}$ &$\lambda_{\text{reg}}$  & Initial $l_r$ & $l_r$ decay rate    \\
\hline
tLaSDI  & 1 &  1e-1  & 1e-2  &1e-7 & N/A & N/A &0&1e-4 & 1\%\\ 
 \hline
TA-ROM & 1e3 &  1  & N/A  & N/A   & 1 & 1e-6 &1e-5& 1e-5/1e-5 & 1\% / 1\%\\ 
\hline
 Vanilla-FNN & 1 &  1e-1  & N/A  & N/A   & N/A & N/A &0&1e-4& 1\%\\ 
 \hline
\end{tabular}
\caption{Example \ref{sec:gas_containers}. Hyperparameters used for implementations of tLaSDI, TA-ROM, and Vanilla-FNN; The learning rate and its decay rate for TA-ROM correspond to SAE/SPNN.}
\label{tab:loss_weights_all_methods_GC}
\end{table}

\subsection{1D Burgers' equation}
\label{sec:app_1DBG}
The implementation details for the tLaSDI and gLaSDI \cite{he2023glasdi} corresponding to the results in Section \ref{sec:1DBG}, are presented in this section.
For the gLaSDI implementation, we followed the algorithm detailed in \cite{he2023glasdi}, utilizing the code\footnote{GitHub page: \url{https://github.com/LLNL/gLaSDI}} provided by the authors.
The batch size of training data is selected so that the number of batches is $4$ for tLaSDI and $2$ for gLaSDI. 
The Runge-Kutta-Fehlberg method (RK23) is employed for the time-step integration of latent dynamics for tLaSDI.
The time integration of latent dynamics for gLaSDI is performed using an ODE solver, odeint function of the Scipy library \cite{virtanen2020scipy}.

{\bf Greedy sampling}: The greedy sampling algorithm is employed for both tLaSDI and gLaSDI. The training initiates with four initial data points collected at the corners of the parameter domain, ${(0.7,0.9), (0.7,1.0), (0.8,0.9), (0.8,1.0)}$.
After every $2,000$ epochs, training is paused, and a new parameter that maximizes a physics-informed residual-based error is selected. 
At this newly identified parameter, a high-fidelity simulation is executed to solve equation (\ref{eq:1DBG}), and the solution is added to the training dataset. 
Then the training is resumed with the updated training dataset. 
This process is repeated until the dataset encompasses $25$ training parameters, including the initial four. 
For more detailed insights into each step, refer to \cite{he2023glasdi}.

{\bf Data generation}: We initially generate solution data with uniform spatial and time steps of $\Delta x = 6/1000$ and $\Delta t = 2/1000$, respectively.
The step sizes of the spatial and temporal discretizations are small enough to ensure high accuracy of the generated solution data. 
The solution data are generated using an implicit Euler time integration method, while the corresponding derivatives are calculated using a backward difference scheme. 
From this dataset, the training data and the corresponding derivatives are subsampled with the resolution of $\Delta x = 6/200$ and $\Delta t = 2/200$.

{\bf Hyperparameters and NN architectures}:
The hyperparameters are selected to yield optimal predictive performance of each method. 
The hyperparameters including loss weights and learning rate scheduling for training tLaSDI and gLaSDI are listed in Table \ref{tab:loss_weights_all_methods_BG}.
For tLaSDI, the reconstruction loss weight is selected from $\{10^{-2}, 10^{-1}\}$. 
Then the Jacobian loss weight is determined from the set $\{10^{-9}, 10^{-8}, \dots, 10^{-2}\}$. 
Finally, model loss weight for tLaSDI is chosen from $\{10^{-8}, 10^{-7}, 10^{-6}, 10^{-5}\}$. 
For gLaSDI, the reconstruction loss weight is fixed at $1$, and the model loss weight is selected from $\{10^{-3}, 10^{-2}, \dots, 10^2\}$.
Initial learning rates for both models are selected from $\{10^{-5}, 10^{-4}, 10^{-3}\}$. 
The learning rate decay is conducted for both autoencoder (hypernetworks) and GFINNs in tLaSDI after every $1000$ epochs. 

We employ 201-100-10 FNN architectures for the encoders of tLaSDI and gLaSDI, and symmetric architectures are applied for the decoders.
The autoencoders of tLaSDI and gLaSDI utilize hyperbolic tangent and Sigmoid activation functions respectively. 
All NNs in GFINNs of tLaSDI consist of $5$ layers and $40$ neurons within each hidden layer and hyperbolic tangent activation functions. 
The hypernetworks for the encoder and decoder of tLaSDI have $3$ layers and $20$ neurons within each hidden layer with hyperbolic tangent activation functions. 

Further specification of hyperparameters is required for gLaSDI.
A quadratic polynomial is employed in the dictionary-based DI model.
The KNN convex interpolation scheme utilizes $1$ and $5$ nearest neighbors to evaluate the model at unseen parameters (beyond the training parameters) during greedy sampling and prediction, respectively.
For a more comprehensive understanding of the hyperparameters of gLaSDI, readers are referred to \cite{he2023glasdi}.

\begin{table}[t]
\centering
\begin{tabular}{|c||c|c|c|c|c|c|}
\hline
Method &$\lambda_{\text{int}}$& $\lambda_{\text{rec}}$ & $\lambda_{\text{Jac}}$& $\lambda_{\text{mod}}$   & Initial $l_r$ & $l_r$ decay rate    \\
\hline
tLaSDI  & 1 &  1e-1  & 1e-9  &1e-7  &1e-4 & 1\%\\ 
 \hline
gLaSDI & N/A &  1  & N/A  & 1  &1e-3& 0\%\\ 
 \hline
\end{tabular}
\caption{Example \ref{sec:1DBG}. Hyperparameters used for implementations of tLaSDI and gLaSDI}
\label{tab:loss_weights_all_methods_BG}
\end{table}

\end{appendices}

\clearpage
\bibliography{reference-tLaSDI,reference-ROM}

\end{document}